\documentclass{article}

\PassOptionsToPackage{numbers, compress}{natbib}

\usepackage{placeins}

\usepackage[utf8]{inputenc} %
\usepackage[T1]{fontenc}    %
\usepackage[hidelinks]{hyperref}       %
\usepackage{url}            %
\usepackage{booktabs}       %
\usepackage{amsfonts}       %
\usepackage{nicefrac}       %
\usepackage{xcolor}         %

\usepackage{natbib}

\usepackage{centernot}
\usepackage{amsthm}

\newtheorem{lemma}{Lemma}

\newtheorem{cor}{Corollary}

\newcommand{\cL}{\mathcal{L}}

\newcommand{\cN}{\mathcal{N}}

\newcommand{\g}{\, | \,}

\usepackage{booktabs,arydshln}
\makeatletter
\def\adl@drawiv#1#2#3{%
        \hskip.5\tabcolsep
        \xleaders#3{#2.5\@tempdimb #1{1}#2.5\@tempdimb}%
                #2\z@ plus1fil minus1fil\relax
        \hskip.5\tabcolsep}
\newcommand{\cdashlinelr}[1]{%
  \noalign{\vskip\aboverulesep
           \global\let\@dashdrawstore\adl@draw
           \global\let\adl@draw\adl@drawiv}
  \cdashline{#1}
  \noalign{\global\let\adl@draw\@dashdrawstore
           \vskip\belowrulesep}}
\makeatother

\newcommand{\bas}[1]{\begin{align*}#1\end{align*}}

\newcommand{\ba}[1]{\begin{align}#1\end{align}}

\newcommand{\norm}[1]{\left\lVert#1\right\rVert}

\newcommand{\idx}[1]{\{1,2,\cdots,#1\}}

\newcommand{\distas}[1]{\mathbin{\overset{#1}{\kern\z@\sim}}}

\newcommand{\indep}{\rotatebox[origin=c]{90}{$\models$}}

\newcommand{\var}[1]{\textrm{Var}\left[#1\right]}

\newcommand{\Dmat}{{\bf D}}

\newcommand{\Imat}{{\bf I}}

\newcommand{\Ymat}[0]{{{\textbf Y}}}

\newcommand{\xv}{\boldsymbol{x}}

\newcommand{\zv}{\boldsymbol{z}}
\newcommand{\cdotv}{\boldsymbol{\cdot}}

\newcommand{\betav}[0]{{\boldsymbol{\beta}}}

\newcommand{\phiv}{\boldsymbol{\phi}}

\newcommand{\psiv}{\boldsymbol{\psi}}

\newcommand{\bR}{\mathbb{R}}

\newcommand{\Z}{\mathcal{Z}}

\newcommand{\bI}{\mathbf{I}}

\usepackage{color, colortbl}
\usepackage{mathtools}
\usepackage{cleveref}
\usepackage{bbm}
\usepackage{algorithm}
\usepackage{algorithmic}
\usepackage{caption}
\usepackage{subcaption}
\usepackage{multirow}

\usepackage[
    textheight=9in,
    textwidth=6.2in,
    top=1in,
    headheight=12pt,
    headsep=25pt,
    footskip=30pt
  ]{geometry}

\usepackage[final,expansion=alltext]{microtype}
\usepackage[english]{babel}
\usepackage[parfill]{parskip}
\usepackage{afterpage}
\usepackage{framed}

\renewenvironment{abstract}%
{%
  \vskip 0.075in%
  \centerline%
  {\large\bf Abstract}%
  \vspace{0.5ex}%
  \begin{quote}%
}
{
  \par%
  \end{quote}%
  \vskip 1ex%
}

\DeclarePairedDelimiter\ceil{\lceil}{\rceil}

\newtheorem{theorem}{Theorem}

\DeclareMathOperator*{\argmin}{arg\,min}
\crefname{cor}{corollary}{corollary}
\Crefname{cor}{Corollary}{Corollary}

\definecolor{hl_color}{gray}{0.9}
\newcommand{\whitecell}{\cellcolor{white}}
\newcommand{\graycell}{\cellcolor{hl_color}}

\def\bb{\textcolor{blue}}

\newcommand*\samethanks[1][\value{footnote}]{\footnotemark[#1]}

\title{Probabilistic Conformal Prediction Using \\Conditional Random Samples}

\usepackage{setspace}

\author{Zhendong Wang \thanks{The first three authors contributed equally.}  \thanks{University of Texas at Austin} \\
   \texttt{zhendong.wang@utexas.edu} \\
  \and
  Ruijiang Gao\samethanks[1] \samethanks[2]\\
   \texttt{ruijiang@utexas.edu} \\
  \and
  Mingzhang Yin\samethanks[1] \thanks{Columbia University}\\
  \texttt{my2674@columbia.edu} \\
  \and
  Mingyuan Zhou \samethanks[2]\\
  \texttt{mingyuan.zhou@mccombs.utexas.edu} \\
  \and
  David M. Blei \samethanks[3] \\ 
  \texttt{david.blei@columbia.edu}\\
}

\begin{document}

\maketitle

\begin{abstract}

This paper proposes probabilistic conformal prediction (PCP), a predictive inference algorithm that estimates a target variable by a discontinuous predictive set. Given inputs, PCP construct the predictive set based on random samples from an estimated generative model. It is efficient and compatible with either explicit or implicit conditional generative models. Theoretically, we show that PCP guarantees correct marginal coverage with finite samples. Empirically, we study PCP on a variety of simulated and real datasets. Compared to existing methods for conformal inference, PCP provides sharper predictive sets.

\end{abstract}

\section{Introduction} \label{sec:intro}

A core problem in supervised machine learning (ML) is to predict a target variable $Y \in \mathcal{Y}$ given a vector of inputs $X \in \bR^p$. In this problem, a predictive function $q(Y \g X)$ is fitted on an observed dataset $\Dmat = \{(X_i, Y_i)\}_{i=1}^N$ and then used to predict the target $Y_{N+1}$ of a new data point with inputs $X_{N+1}$. While much of machine learning focuses on point predictions of $Y$, the problem of predictive inference aims at more robust prediction.  In predictive inference, our goal is to create a \textit{predictive set} that is likely to contain the unobserved target \citep{geisser1993predictive}. %

In particular, the field of \textit{conformal inference} develops predictive inference algorithms that aim for calibrated coverage probabilities \citep{papadopoulos2002inductive,vovk2005algorithmic}.  Assume the data pairs $(X_i, Y_i)$ are sampled independent and identically distributed (iid) from a population distribution $\mathbb{P}(X,Y)$.  Given an input $X$, a conformal inference algorithm provides a set $C_{\alpha}(X)$ such that
\ba{ 
\mathbb{P}_{X,Y}(Y \in \hat{C}_{\alpha}(X)) \geq 1-\alpha.
\label{eq:coverage}
}
The scalar $\alpha \in [0,1]$ is a predefined miscoverage rate and $\hat{C}_{\alpha}(X) \subset \mathcal{Y}$ is the predictive set. A set that  satisfies \Cref{eq:coverage} is called a \textit{valid} predictive set. Since the trivial set $\hat{C}_{\alpha}(X) = \mathcal{Y}$ is valid, one goal of conformal inference is to keep the size of the predictive set small and (thus) informative. This property is known as \textit{sharpness} \citep{dorn2022sharp,lei2015conformal}. In this paper, we develop a new method for conformal inference that produces valid and sharp predictive sets.

Existing conformal inference methods often produce a continuous interval as the predictive set \citep{Barber2019-rl,Chernozhukov2021DistributionalCP,lei2014distribution,messoudi2021copula,romano2019conformalized,Sesia2021-ei}. Such intervals are appropriate in some predictive situations.  However, consider a target distribution with separated high-density regions. In this setting, validity comes at the cost of sharpness~\citep{Hoff2021-hk}: to ensure validity the set must include all of the high-density regions; but since it's continuous it must also include the low-density regions between them.

\begin{figure}[t]
    \centering
     \begin{subfigure}[b]{0.19\textwidth}
         \includegraphics[width=\textwidth]{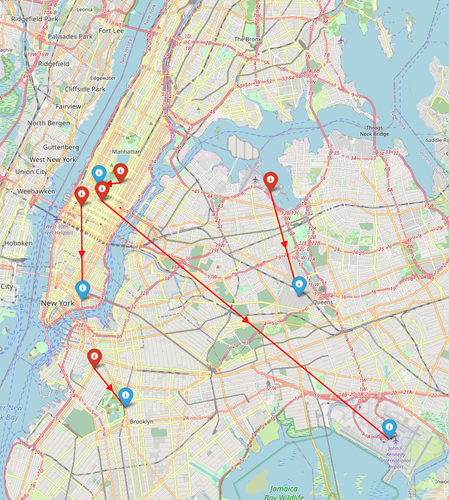}
     \caption{Sample Data}
     \end{subfigure}     
     \begin{subfigure}[b]{0.19\textwidth}
         \includegraphics[width=\textwidth]{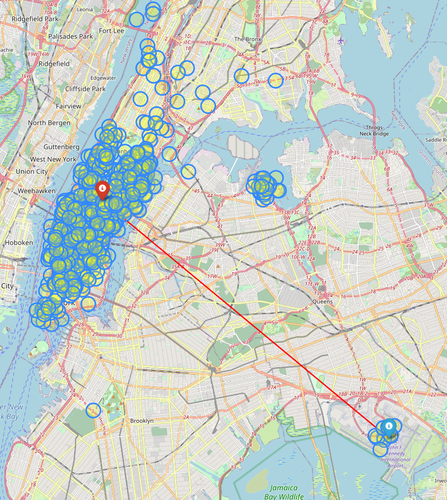}
     \caption{PCP (Ours)}
     \end{subfigure} 
     \begin{subfigure}[b]{0.19\textwidth}
         \includegraphics[width=\textwidth]{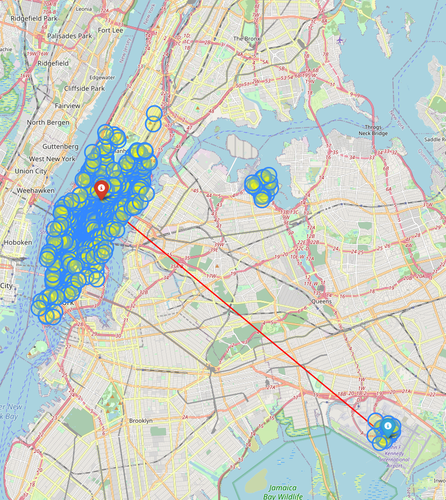}
     \caption{HD-PCP (Ours)}
     \end{subfigure}
     \begin{subfigure}[b]{0.19\textwidth}
         \includegraphics[width=\textwidth]{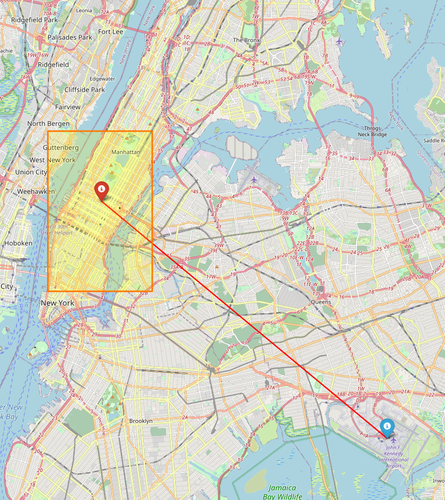}
     \caption{CDSplit \citep{izbicki2020flexible}}
     \end{subfigure}
     \begin{subfigure}[b]{0.19\textwidth}
         \includegraphics[width=\textwidth]{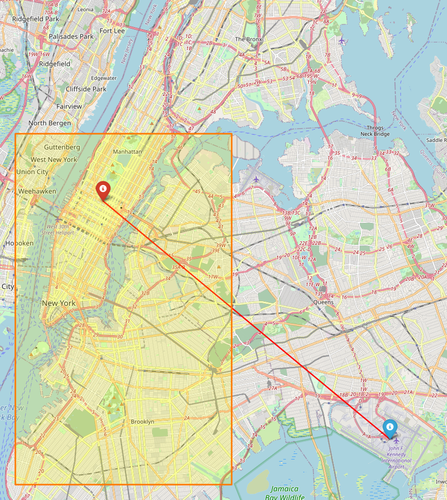}
     \caption{CHR \citep{Sesia2021-ei}}
     \end{subfigure} 
    \caption{NYC Taxi data. The covariates are pickup location (red pin) and other passenger information; The targe is the dropoff location (blue pin). Left to right: Five random samples from Travel Data; Predictive sets output by PCP, HD-PCP, CDSplit \citep{izbicki2020flexible} and CHR \citep{Sesia2021-ei} for one travel record.}
    \label{fig:taxiaddtion}
\end{figure}

For example, consider a prediction problem that estimates the drop-off location of a taxi passenger based on the passenger's information.  The target distribution is likely to be multimodal, centered around locations such as tourist attractions and transit centers \citep{Trippe2018-vp}. A continuous predictive set will have to emcompass these regions, regardless of how far apart they are.  A more informative set would contain the regions themselves, but not the areas between them. Other examples of multimodal targets include the effects of a stroke on brain regions \citep{gillmann2021visualizing}, and the rewards of  actions of a robot \citep{myers2022learning}.

\Cref{fig:taxiaddtion} provides an example of the taxi application. Panel (a) illustrates the data, the destinations of New York taxi passengers.  Given a new set of inputs, panels (d) and (e) show existing methods for conformal inference, which predict large regions for the possible destination.  Panels (b) and (c) show the results of our algorithms, which form sharper predictive sets from distinct subregions of the map.  Our method is called \textit{probabilistic conformal prediction} (PCP). \Cref{fig:flow_chart} illustrates the algorithm.

In more detail, PCP builds on the split conformal prediction framework \citep{lei2014distribution,papadopoulos2002inductive}. It begins by randomly splitting the observed data $\Dmat$ into a preliminary set $\Dmat_{\text{pre}}$ and a calibration set $\Dmat_{\text{cal}}$.  It then has three stages. (1) It fits a conditional generative model $q(Y \g X)$ to the preliminary  data $\Dmat_{\text{pre}}$.  (2) For each point $(X_i, Y_i)$ in the calibration set $\Dmat_{\text{cal}}$, it generates $K$ independent samples of preditions $\hat{\Ymat}_{X_i}= \{\hat{Y}_{i1}, \cdots, \hat{Y}_{iK}\}$  from the fitted model $q(Y \g X_i)$. It then calculates the distance between each sampled predition and the true label $Y_i$. These quantities are called the \emph{nonconformity scores} and measure the goodness-of-fit of the generative model.  (3) Finally, it calculates and records the $(1-\alpha)$ empirical quantile of the nonconformity scores. These will be used to for the predictive sets.

With these calculations in place, PCP can form the predictive set of a new datapoint.  First it generates sampled predictions from the fitted target distribution. Then each sample is expanded to a ball that centers at its  point and has a radius equal to the quantile computed from the calibration set. Finally, the predictive set is defined as the union of the balls over the samples. Because it is centered at high-density regions, this predictive set is sharp.  Further, as we prove below, it is valid.

There are several advantages to PCP (and a related extension, high-density PCP). First, it adapts automatically to the landscape of the target distribution, providing sharp and valid predictive sets regardless of the underlying distribution.  Second, the generative model for PCP may have an explicit or implicit density function as long as the random samples can be generated from it. Without requiring an explicit density, PCP is compatible with the likelihood-free prediction \citep{alsing2019fast,chan2018likelihood} and is less prone to model misspecification  \citep{mirza2014conditional,Yin2018SemiImplicitVI}. Last, (HD-)PCP can be applied to multi-target regression where the target variable $Y \in \mathcal{Y} = \bR^T$, $T \geq 1$ \citep{breiman1997predicting,messoudi2021copula}. As we shall see, (HD-)PCP scales efficiently with the target dimension and creates a sharp predictive set by capturing the targets' dependencies.

\paragraph{Related Work.} PCP provides a contribution to the growing field of conformal inference. Some conformal inference methods are based on predicting summary statistics of the target distribution, for example, by fitting a mean response function \citep{lei2018distribution,shafer2008tutorial}, conditional quantile functions \citep{romano2019conformalized} and approximate histograms \citep{Sesia2021-ei}. However, these methods produce a single continuous interval as the predictive set, which might be too loose for predicting multimodal targets. %

Other conformal inference algorithms estimate the full target distribution. Distributional conformal prediction (DCP) is based on the estimated cumulative density function \citep{Chernozhukov2021DistributionalCP} but its prediction is often sensitive to the tail estimation \citep{Sesia2021-ei}. CDSplit uses a level set of the estimated probability density function as the predictive set \citep{izbicki2020flexible}. Similar to PCP, CDSplit can produce discontinuous predictive sets. However, the level set might be loose when the distribution has high dispersion and it has to be computed approximately. 
Thanks to its sampling-based design, PCP is more computationally efficient than CDSplit, and further it is compatible with likelihood-free predictions \citep{alsing2019fast}. Empirically, across multiple datasets, PCP creates sharper predictive sets than these existing conformal methods.

Finally, there are a few conformal methods for multi-target regression~\citep{messoudi2020conformal,messoudi2021copula,neeven2018conformal}. 
Compared to these methods, PCP models the target variables jointly and can produce discontinuous predictive sets. As we show in the empirical studies, PCP provides sharper and more interpretable predictions.

\begin{figure}[t]
    \centering
    \includegraphics[width=\textwidth]{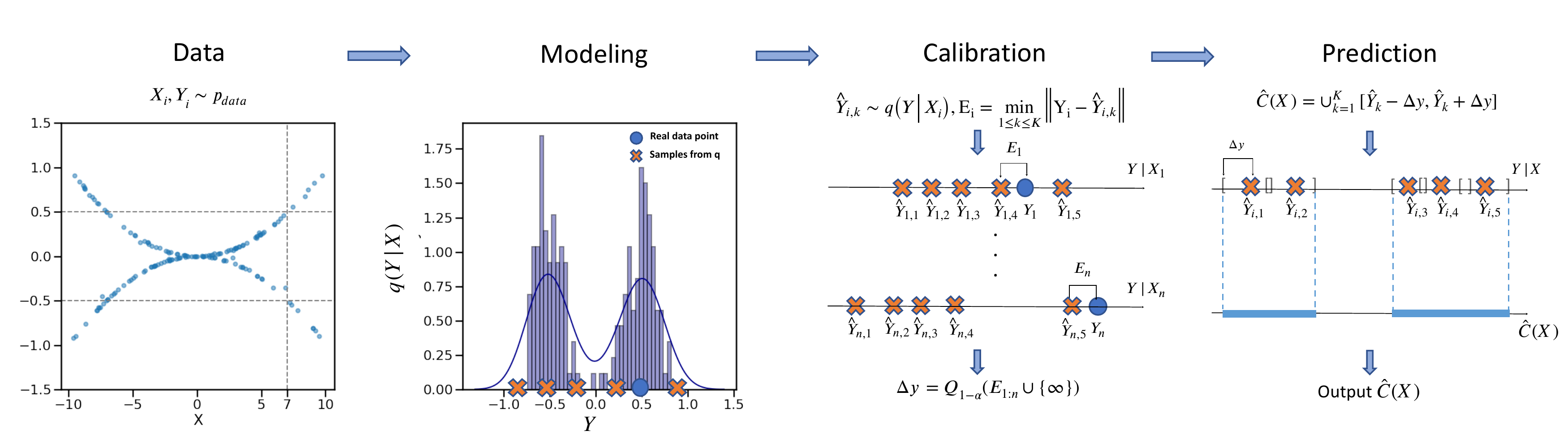}
    \caption{Illustration of the stages of PCP. Data: i.i.d data $\Dmat = \{(X_i, Y_i)\}_{i=1}^N$; Modeling: generate K random samples from a fitted $q(Y\g X)$; Calibration: compute scores $E_i$ and the quantile $\Delta y$;  Prediction: create the predictive set $\hat{C}(X)$ for a test data.}
    \label{fig:flow_chart}
\end{figure}

\section{Probabilistic conformal prediction}

\subsection{Problem setup}

Consider i.i.d. pairs of covariates $X_i$ and a target variable $Y_i$, i.e. $\Dmat = \{(X_i, Y_i)\}_{i=1}^N$, from an underlying distribution. We observe data $\Dmat$ and the covariates $X_{N+1}$ of a new data point. The goal is to form a predictive set $\hat{C}(X_{N+1})$ for the unobserved target $Y_{N+1}$ with valid uncertainty estimation. Specifically, we create a predictive set $\hat{C}_{\alpha}(\cdot): \mathcal{X} \mapsto \mathcal{Y}$ that satifies \Cref{eq:coverage} for $\alpha \in [0,1]$. %
Since an arbitrary wide predictive set has valid coverage, a predictive set should be as sharp as possible. %

Classic conformal prediction is based on leave-one-out estimation \citep{vovk2005algorithmic}, which has high computational cost due to multiple model fitting. In this paper, we adopt the split conformal prediction framework, which improves computational efficiency by data-splitting \citep{lei2018distribution,papadopoulos2002inductive}. It randomly splits the observed data to a preliminary set and a calibration set. The model is fit on the preliminary set and kept fixed in computing the nonconformity scores on the calibration set and the test set. %

\subsection{Generative model fitting} \label{sec:cde}

Our proposed PCP depends on random samples from a conditional generative model $q(Y | X)$ that approximates the target variable distribution $p(Y|X)$. This differs from standard conformal prediction methods that are based on fitting the summary statistics such as the conditional mean and quantiles of the target \citep{lei2014distribution,romano2019conformalized} and that depend on evaluating probability densities \citep{Chernozhukov2021DistributionalCP,Hoff2021-hk,izbicki2020flexible}. 
Since the only prerequisite is to sample from $q(Y|X)$, we consider both typical conditional density estimation methods with explicit density function and popular generative models with implicit density. %

PCP is compatible with a variety of CGMs, such as the Kernel Mixture Network (KMN) \citep{Ambrogioni2017}, Mixture Density Network (MixD) \citep{bishop1994mixture}, Quantile Regression Forest (QRF) \citep{Meinshausen2006QuantileRF} and implicit generative models. %
We refer to \Cref{sec:appendix_cde} for more details about fitting a CGM and generating random samples from it. We regard CGMs as backbone models for PCP.

\subsection{Uncertainty calibration with random samples} \label{sec:methodology}

Suppose a conditional density model is fit on a preliminary  data set $\Dmat_{\text{pre}}$. We use the fitted model $q(Y | X)$ and the calibration data to construct a predictive set for a new test data point. %
For a data point $(X_i, Y_i)$ in the calibration set, the algorithm first generates $K$ random samples $\hat{Y}_{ik}$, $k=1,\cdots,K$ independently from $q(Y|X_i)$, denoted as $\hat{\Ymat}_i = \{\hat{Y}_{i1}, \cdots, \hat{Y}_{iK}\}$. Then, it computes the distance from the observed outcome to this set of samples as
\ba{ 
E_i = \min_{1\leq k \leq K}~\norm{Y_i - \hat{Y}_{ik}}.
\label{eq:score}
}
The scalar $E_i$ is set as the nonconformity score. Intuitively, a small score indicates that the speculated outcomes $\hat{Y}_{ik}$ are close to the observed outcome $Y_i$, where $\hat{Y}_{ik}$ are from the approximate density $q(Y|X_i)$ and $Y_i$ is from the true underlying density $p(Y|X_i)$. Therefore, the scale of $E_i$ reflects the distance between the estimated density and the true density. 
We use the empirical quantile of the nonconformity scores to construct the predictive set. The $\alpha$-th empirical quantile is defined as $Q_{\alpha}(z_{1:n}) = \inf_{x} \{(\sum_{i=1}^n \mathbbm{1}[z_i \leq x])/n \geq \alpha\}$ where $\alpha \in [0,1]$ and $\mathbbm{1}[\cdot]$ is the indicator function. 

For a new data point with covariates $X$, we generate $\hat{\Ymat} = \{\hat{Y}_{1}, \cdots, \hat{Y}_{K}\}$ with $\hat{Y}_k \sim q(Y|X)$. Suppose that the desired nominal coverage is $1-\alpha$. Then, each sample $\hat{Y}_k$ is expanded to a region $R_{k} = \{y: \norm{y - \hat{Y}_{k}}\leq r\} $ with an arbitrary norm and a radius $r$  as the $(1-\alpha)$ quantile of the scores $\{E_1, \cdots, E_n\} \cup \{\infty\}$. We call $R_{k}$ an element region of the data point $X$. The proposed predictive set is the union of the element regions, ~\looseness=-1
\ba{ 
\hat{C}(X, \hat{\Ymat}) = \cup_{k=1}^K R_{k} = \cup_{k=1}^K \big\{y: \norm{y - \hat{Y}_k}\leq Q_{1-\alpha}(E_{1:n}\cup \{\infty\}\big\}. 
\label{eq:interval}
}
As a special case, when the outcome is a scalar, the predictive set can be written explicitly as 
\ba{ 
\hat{C}(X, \hat{\Ymat}) = \cup_{k=1}^K \Big[\hat{Y}_k - Q_{1-\alpha}(E_{1:n}\cup \{\infty\}), ~\hat{Y}_k + Q_{1-\alpha}(E_{1:n} \cup \{\infty\})\Big].
\label{eq:interval1}
}
The proposed PCP algorithm is summarized in \Cref{alg:pcp}.

As shown in \Cref{eq:interval}, the predictive set can be either continuous or discontinuous. Therefore, it can produce a sharp estimate by automatically adapting to the distributional properties of the target distribution.  When the generative model is less well fitted, PCP maintains a valid marginal coverage, properly quantifying the predictive uncertainty. When the generative model fits the target distribution well, the predictive set allocates its volume according to the random samples. For example, if  $p(Y|X)$ is multimodal and the multimodality  is captured by the estimated $q(Y|X)$,  the predictive set would consist of discontinuous sets around the modes where each set is relatively small. 

Though in some situations, a continuous interval prediction is preferred in terms of interpretability \citep{Sesia2021-ei},  when the target is multimodal, a discontinuous set might be more interpretable. For example, when predicting a watch price based on its appearance without knowing the brand, a price range $(\$100,\$200) \cup (\$1000,\$1200)$ might be more informative  than $(\$100,\$1200)$. Nevertheless, one can take the convex hull of a discontinuous set to form a continuous interval but not vice versa.  

By the construction of the predictive set defined in \Cref{eq:interval},  the estimated density $q(Y|X)$ can be explicit or implicit. %
This flexibility makes PCP compatible with a wide range of density estimators and generative models. %
Moreover, the predictive set in \Cref{eq:interval} can be computed without approximation, making PCP scalable to a high dimensional target variable $Y$.  %
Finally, PCP has a guaranteed  marginal coverage as shown in \Cref{thm:cov}.

\begin{algorithm}[!t]
  \caption{Probabilistic Conformal Prediction} %
  \label{alg:pcp}
\textbf{Input: } Data $\Dmat = \{(X_{i}, Y_{i})\}_{i=1}^{N}$, model $q(Y|X)$, nominal level $\alpha$, test point $X$, sample size $K$. 

  \vspace{0.2cm}
  \textbf{Step I: Conditional generative model}
  \begin{algorithmic}[1]
    \STATE Split the data into three folds $\Z_{\text{tr}}$, $\Z_{\text{val}}$, $\Z_{\text{cal}}$ with set of index as $\bI_{\text{tr}}$, $\bI_{\text{val}}$, $\bI_{\text{cal}}$ respectively
    \STATE Fit  $q(Y|X)$ on $\Z_{\text{tr}}$  with hyper-parameter chosen by cross validation on $\Z_{\text{val}}$
  \end{algorithmic}
  \vspace{0.2cm}
  \textbf{Step II: Predictive set for a test point}
    \begin{algorithmic}[1]
    \STATE For $i \in \bI_{\text{cal}} $, sample $\hat{Y}_{i1}, \cdots, \hat{Y}_{iK} \sim q(Y|X_i)$ 
    
    \STATE For test point $X$, sample $\hat{Y}_{1}, \cdots, \hat{Y}_{K} \sim q(Y|X)$
    
     \STATE Compute nonconformity score $ \{E_i\}_{i \in {\bI}_{\text{cal}} }$ by \Cref{eq:score}, $E_{N+1}=\infty$, $\tilde{\bI}_{\text{cal}} = \bI_{\text{cal}} \cup \{N+1\}$
     
     \STATE Set $r$ as the $(1-\alpha)$ empirical quantile of $ \{E_i\}_{i \in \tilde{\bI}_{\text{cal}} }$
 
 \STATE Compute the predictive set $\hat{C}(X, \hat{\Ymat})$ by \Cref{eq:interval} 
  \end{algorithmic}
  \textbf{Output:} Predictive set $\hat{C}(X, \hat{\Ymat})$ %
\end{algorithm}

\begin{theorem}
\label{thm:cov}
Suppose  $(X_i, Y_i)$, $i \in \idx n$ and $(X, Y)$ are exchangeable, then 
\begin{enumerate}
\item[(1)] the predictive set in \Cref{eq:interval} satisfies
\ba{ 
\mathbb{P}_{X,Y, \hat{\Ymat}}(Y \in \hat{C}(X, \hat{\Ymat})) \geq 1-\alpha;
\label{eq:lb}
}
\item[(2)] when the scores $E_1, \cdots, E_n$ are distinct almost surely, 
\ba{ 
\mathbb{P}_{X,Y, \hat{\Ymat}}(Y \in \hat{C}(X, \hat{\Ymat})) \leq 1-\alpha + \frac{1}{n+1}. 
\label{eq:ub}
}
\end{enumerate}
\end{theorem}

\Cref{thm:cov} demonstrates that the marginal coverage of  PCP  is tight. In particular, the condition of the upper bound is satisfied when $p(Y|X)$ is continuous with respect to the Lebesgue measure. %

In practice, we take the quantile of $E_{1:n}$ instead of the inflated scores $E_{1:n}\cup \{\infty\}$ in \Cref{eq:interval}. The following corollary offers the coverage guarantee under such modification. 

\begin{cor} 
Under the conditions of \Cref{thm:cov} and suppose $\alpha \geq 1/(n+1)$, if the quantile in \Cref{eq:interval} is $Q_{1-\alpha}(E_{1:n})$,  then $\mathbb{P}(Y \in \hat{C}(X, \hat{\Ymat})) \in [1-\alpha - 1/(n+1), 1-\alpha + 1/(n+1)]$; if the quantile in \Cref{eq:interval} is $Q_{(1-\alpha) (1+\frac1n)}(E_{1:n})$, then $\mathbb{P}(Y \in \hat{C}(X, \hat{\Ymat})) \in [1-\alpha, 1-\alpha +1/(n+1)]$
\label{cor:cov}
\end{cor}

By \Cref{cor:cov}, if we take the $1-\alpha$ quantile of $E_{1:n}$, we may lose the coverage probability by $1/(n+1)$ where $n$ is the number of data in the calibration set. Instead, we can take the $(1-\alpha) (1+\frac1n)$ quantile of $E_{1:n}$ to maintain $(1-\alpha)$ coverage when $\alpha \geq 1/(n+1)$.

\textbf{High Density Probabilistic Conformal Prediction (HD-PCP).} Ideally, we may want the predictive sets to contain only high density regions to offer informative predictions. %
As shown in~\Cref{sec:hdpcpfigure}, for different sets covering a specific probability of a multimodal distribution with the same marginal coverage, the high density region has the smallest size. 

In PCP, the generated random samples include low density samples. When $K$ increases, the set size of the low density region will decrease. However, PCP may generate many isolated sets and make interpretation difficult for practitioners. 
To mitigate this problem, we propose High Density Probabilistic Conformal Prediction (HD-PCP) to 
filter out $\beta$ fraction low-density samples to identify the high density regions when $q(Y|X)$ is explicit. %
Instead of sampling $K$ samples from $q(Y|X)$ like in PCP, we sample $\ceil*{K/(1-\beta)}$ samples for each $X$, keep $1-\beta$ fraction of samples with the highest estimated density, and keep all other parts of the algorithm the same as PCP. %
The HD-PCP algorithm is summarized in Appendix \Cref{alg:hd_pcp}.
The marginal coverage guarantee still holds for HD-PCP, as shown in Corollary~\ref{cor:hdpcp}.  ~\looseness=-1
\begin{cor} 
Under the conditions of \Cref{thm:cov}, HD-PCP has the same marginal coverage as PCP. 
\label{cor:hdpcp}
\end{cor}
The proofs of the theorems and corollaries are presented in \Cref{sec:proof}.

\section{Experiments} \label{sec:experiment}

In this section, we conduct a comprehensive analysis demonstrating the advantages of PCP compared to previously proposed conformal inference methods. We aim to answer the following \textit{questions}: \textbf{(a)} how does PCP perform in terms of coverage and predictive set size when compared with baseline models on synthetic datasets, that has  multimodal $p(Y|X)$ distributions? \textbf{(b)} Does the filtering technique improve the predictive set of PCP? \textbf{(c)} How well do PCP and HD-PCP perform on real datasets with a single target? \textbf{(d)} How do the backbone models impact the performance of PCP?  \textbf{(e)} Does PCP provide better predictive sets in tasks with multi-dimensional targets? The code for the simulations are available at \url{https://github.com/Zhendong-Wang/Probabilistic-Conformal-Prediction}.

Our experiments are structured into three sections. We first conduct experiment on classic 2D synthetic data to answer question \textbf{(a)} and \textbf{(b)}. Then, we compare PCP and HD-PCP with a full set of baseline methods on several selected real datasets to address question \textbf{(b)}, \textbf{(c)} and \textbf{(d)}. Finally, we conduct experiments on multi-dimensional regression tasks to address question \textbf{(e)}. We run all our experiements on machines with AMD EPYC 7763 CPUs.%

\textbf{Baselines.} we consider CHR \citep{Sesia2021-ei}, DistSplit \citep{izbicki2020flexible}, CDSplit \citep{izbicki2020flexible}, DCP \citep{Chernozhukov2021DistributionalCP}, and CQR \citep{romano2019conformalized} as our comparison baselines. For CHR, we use two different conditional density estimation models based on neural network model and random forest model, and we denote them as CHR-NN and CHR-QRF. We evaluate all the mentioned baselines based on their public Github implementation except for CDSplit. We implement python-based CDSplit based on their official R implementation to use the same backbone generative model for a fair comparison, denoted as CDSplit-KMN and CDSplit-MixD.

\textbf{Choosing the hyperparameter K.} 
We conduct an ablation study on the effect of the sample size $K$ of PCP. 
As shown in \Cref{fig:ablation_K}, empirically we find when K increases, the average size of the predictive sets reduces fast first and then gets slow. In practice, we set K moderately large to balance the sharpness and the computational cost, i.e., $K=40$ or $K=1000$ (two-dimensional targets).

\begin{figure}
    \centering
    \includegraphics[width=\textwidth]{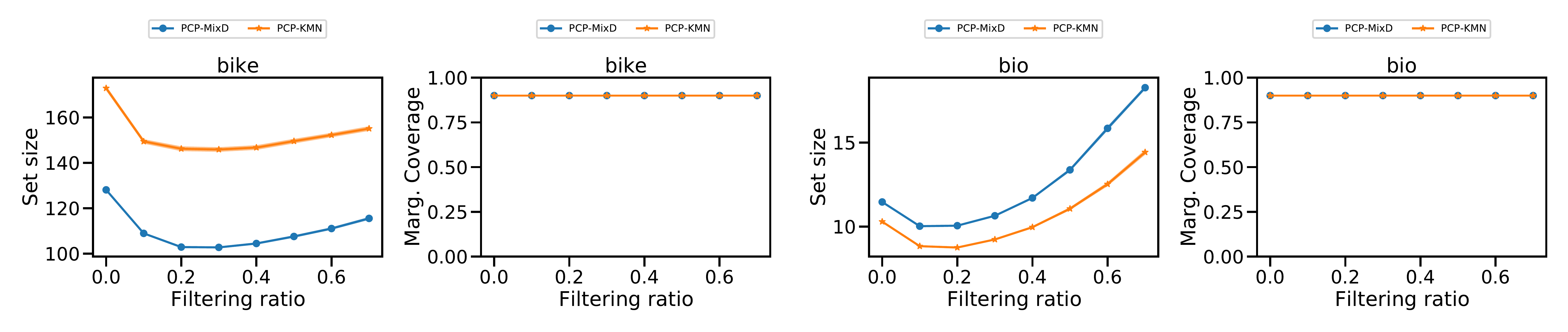} \\
    \includegraphics[width=\textwidth]{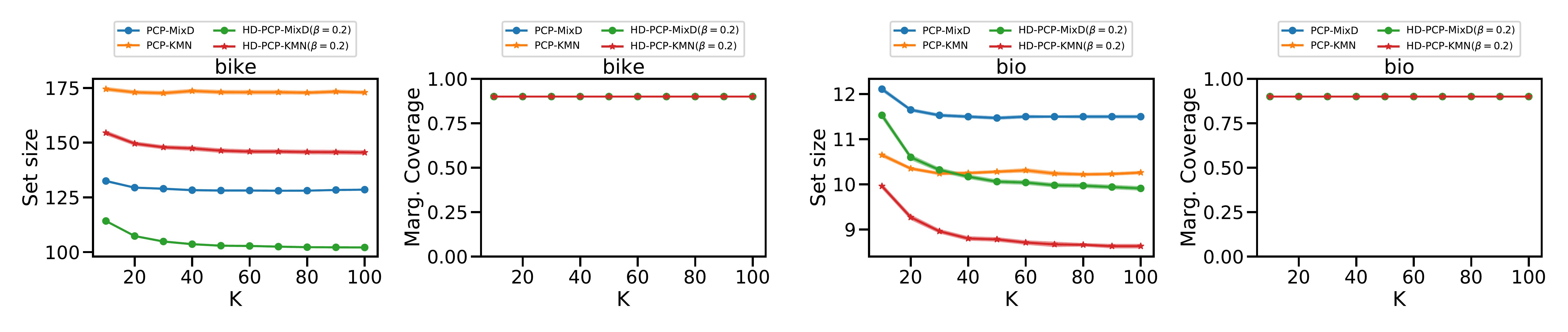}
    \caption{Ablation study on choosing hyperparameter K and filtering ratio $\beta$. We run experiments of PCP and HD-PCP with MixD and KMN as backbone models on two datasets, bike and bio. The K grid is $[10, 20, 30, \dots, 100]$ and the $\beta$ grid is $[0.1, 0.2, \dots, 0.7]$. In the first row, we show predictive set sizes with $K=50$ and varying $\beta$. Second row shows  how predictive set size varies with $K$.}
    \label{fig:ablation_K}
\end{figure}

\subsection{Synthetic data experiments} \label{sec:synthetic_exp}

To evaluate the effectiveness of proposed methods, we compare the predictive set of PCP and HD-PCP with other baseline methods on classic 2D synthetic data: s-curve, half-moons, 25-Gaussians, 8-Gaussians, circle and swiss-roll. We show the evaluation results of the s-curve and the 25-Gaussians in \Cref{fig:synthetic} and place detailed results in Appendix \ref{appendix:synthetic}.

\Cref{fig:synthetic} illustrates, for dataset that has multimodal $p(Y|X)$ distribution, models that consider multimodality, such as CDSplit and (HD-)PCP, works apparently better than the models that can only provide unimodal predictions. This is consistent with our discussion in \Cref{sec:intro}. %
Quantitatively, all models achieve the target marginal coverage ($1-\alpha$), while the average set sizes of CDSplit and (HD-)PCP are several times smaller than that from CHR. We demonstrate that CDSplit and PCP both can provide sharp and informative predictive sets for these multimodal datasets and PCP is slightly better with respect to the set size. The right two panels show the effect of filtering high density samples. The predictive sets from HD-PCP become cleaner and concentrated on the correct modes. Correspondingly, the average set sizes of HD-PCP drop a lot compared to PCP. The histogram moving from blue bins to orange bins also demonstrates the effectiveness of the filtering.

\begin{figure}[!t]
    \centering
    \includegraphics[width=\textwidth]{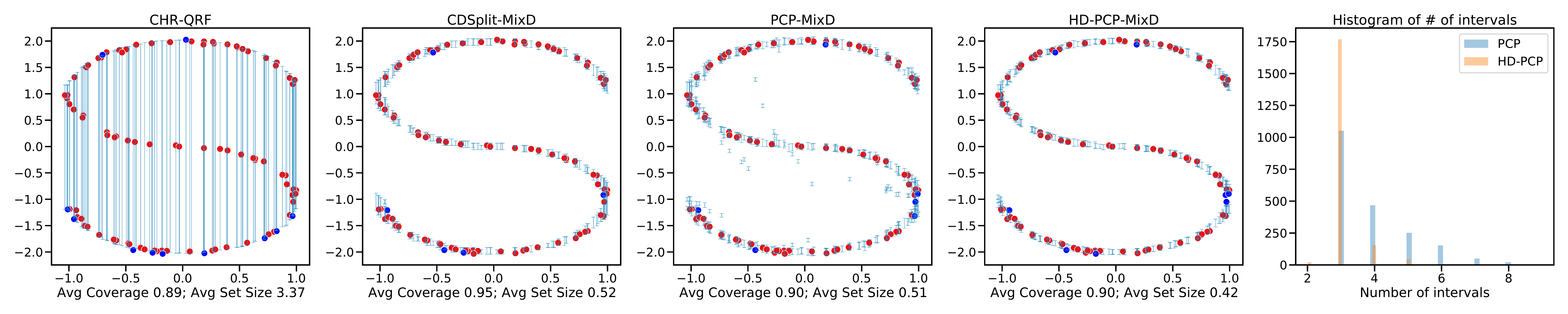} \\
    \includegraphics[width=\textwidth]{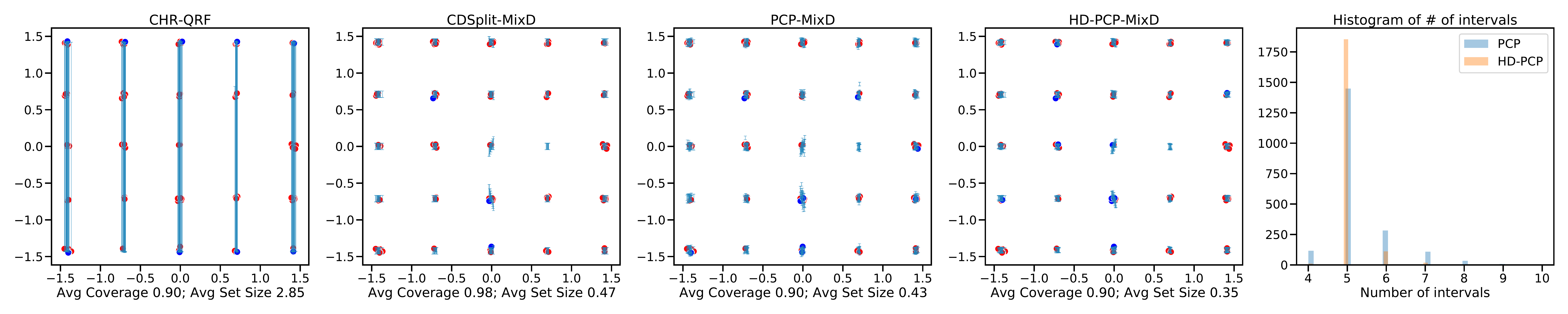} \\
    \caption{ Visualization of predictive sets ($\alpha=0.1$) on 2D toy datasets: s-curve and 25-Gaussians. We shows the predictive sets on 100 test data samples. Blues lines: the predictive sets from each method; Blue dots: test points that are not covered by the predictive sets; Reds dots: test points covered. We report the marginal coverage and the average set size across test datapoints in the x-axis label. The fifth column shows the histogram of the number of predicted intervals of PCP and HD-PCP. We set $K=40$ for (HD-)PCP, $\beta=0.2$. Detailed experiments are in \Cref{appendix:synthetic}.}
    \label{fig:synthetic}
\end{figure}

\subsection{Real data experiments}

We study regression tasks on several real data sets to evaluate the effectiveness of PCP and HD-PCP \citep{Sesia2021-ei}. We consider multiple types of generative models $q(Y \g X)$, including implicit models (GAN \cite{Goodfellow2014GenerativeAN}) and semi-implicit model \citep{wang2020thompson,Yin2018SemiImplicitVI,yin2019semi}, explicit models (KMN \citep{Ambrogioni2017}, MixD \citep{bishop1994mixture}), and QRF \cite{Meinshausen2006QuantileRF, Sesia2021-ei}. We denote them as PCP-GAN, PCP-SIVI, PCP-KMN, PCP-MixD and PCP-QRF respectively. 

\textbf{Datasets.} We conduct real data experiments on 9 public-domain data sets: bike sharing data (bike), physicochemical properties of protein tertiary structure (bio), blog feedback (blog), and Facebook comment
volume, variants one (fb1) and two (fb2), medical expenditure panel survey number 19 (meps19), number 20 (meps20), and number 21
(meps21) \citep{romano2019conformalized} and temperature forecast data \citep{Cho2020ComparativeAO}.  \Cref{tab:dataset_splits} in \Cref{sec:full_real_exp} illustrates the dataset sizes and data splits for training, calibration and testing. 

\textbf{Evaluation Protocol.} We evaluate the marginal coverage, conditional coverage (approximated by the worst-slab conditional coverage \cite{Cauchois2021KnowingWY,Romano2020ClassificationWV}), and the predictive set size.%
We report results based on 50 random %
splits for all datasets. 

\Cref{tab:real_data} shows numerical results. For our methods, we report PCP-MixD and HD-PCP-MixD; for baselines, we report the backbone model that works generally the best across the 9 datasets with respect to set size in the main paper. 
See detailed results in \Cref{sec:full_real_exp}: \Cref{tab:real_summary} reports the best results among the variants of each method; \Cref{tab:real_data_1}, \Cref{tab:real_data_2} and \Cref{tab:real_data_3} report full experiment results. 

We  observe that all conformal methods achieve  $(1 - \alpha)$ marginal coverage and perform well in terms of the worst-slab conditional coverage. Thus, our comparison focuses on the size of predictive sets. As shown in \Cref{tab:real_data}, HD-PCP-MixD outperforms all the other baselines on 7 out of 9 datasets in terms of the predictive set size. When picking a backbone model for each dataset is allowed, our models outperform baselines on all datasets, as shown in \Cref{tab:real_summary}. Comparing HD-PCP with PCP, we find that the filtering technique leads to consistent performance improvement. \Cref{tab:real_summary} shows that PCP outperforms the baselines by a large margin especially on blog, facebook1 and facebook2 datasets. 

Moreover, note that PCP is flexible for backbone generative models, as long as the model can easily generate random conditional samples. %
The flexibility of PCP makes it achieve good performance by choosing a proper generative model according to data sets. For example, PCP-SIVI works well in bike and facebook data with implicit backbone conditional generative models, as shown in \Cref{tab:real_data_1}.

The limitation of CDSplit may be explained by noting that the method needs to make partitions of data based on K-means algorithm,  which is known to be unstable due to local minima. 
It also needs to approximate the level set on a grid of the target space to form the predictive set. It may be sensitive to the range and coarseness of the grid.    
Thus, we notice that CDSplit produces large predictive sets on facebook data and meps data. Similar to \citet{Sesia2021-ei}, we observe that DCP 
is sensitive to the estimation of the distribution tails \citep{Sesia2021-ei}, which makes it unstable for some datasets. 
CHR is more robust because it only needs to estimate a histogram with relatively few bins \citep{Sesia2021-ei}, while it could only provide a single continuous interval and produces a loose predictive set when the data exhibits mulimodality. 
CQR predicts intervals based on learned lower and upper quantiles, which leads to large intervals when the learned quantiles are not accurate or the data distribution is multimodal.

\begin{table}[ht]
    \centering
    \resizebox{\textwidth}{!}{
    \begin{tabular}{l|l|c|c|c|c|c|c|c}
         \toprule
         Data & Metric & PCP & HD-PCP & CHR \citep{Sesia2021-ei} & DistSplit \citep{izbicki2020flexible} & CDSplit \citep{izbicki2020flexible} & DCP \citep{Chernozhukov2021DistributionalCP} & CQR \citep{romano2019conformalized} \\
         \midrule
         \multirow{3}{*}{bike} & Marg. C & 0.90 & \graycell 0.90 & 0.90 & 0.90 & 0.92 & 0.90 & 0.90 \\
         & Cond. C & 0.86 & \graycell 0.88 & 0.88 & 0.87 & 0.91 & 0.88 & 0.89 \\
         & Set Size & 128.13(0.53) & \graycell \textbf{102.92(0.48)} & 204.10(1.03) & 423.13(1.51) & 115.74(0.50) & 443.76(1.36) & 403.88(0.86) \\
         \midrule
         \multirow{3}{*}{bio} & Marg. C & 0.90 & 0.90 & 0.90 & 0.90 & \graycell 0.90 & 0.90 & 0.90 \\
         & Cond. C & 0.89 & 0.90 & 0.89 & 0.89 & \graycell 0.90 & 0.89 & 0.89 \\
         & Set Size & 11.47(0.04) & {10.06(0.05)} & 10.21(0.04) & 13.19(0.04) & \graycell \textbf{9.58(0.04)} & 12.95(0.04) & 13.00(0.02) \\
         \midrule
         \multirow{3}{*}{blogdata} & Marg. C & 0.89 & \graycell 0.90 & 0.90 & 0.90 & 0.96 & 0.90 & 0.90 \\
         & Cond. C &  0.85 & \graycell 0.87 & 0.87 & 0.87 & 0.95 & 0.88 & 0.87 \\
         & Set Size &  10.78(0.17) & \graycell \textbf{9.44(0.19)} & 10.81(0.17) & 16.27(0.23) & 39.00(0.40) & 1422.36(0.03) & 15.15(0.26) \\
         \midrule
         \multirow{3}{*}{facebook1} & Marg. C & 0.90 & \graycell 0.90 & 0.90 & 0.90 & 0.95 & 0.90 & 0.90 \\
         & Cond. C &  0.82 & \graycell 0.85 & 0.86 & 0.89 & 0.95 & 0.89 & 0.88 \\
         & Set Size &  9.99(0.14) & \graycell \textbf{8.93(0.12)} & 11.21(0.12) & 14.03(0.16) & 33.69(0.16) & 1303.01(0.04) & 13.79(0.15) \\
         \midrule
         \multirow{3}{*}{facebook2} & Marg. C & 0.90 & \graycell 0.90 & 0.90 & 0.90 & 0.97 & 0.90 & 0.90 \\
         & Cond. C & 0.82 & \graycell 0.84 & 0.87 & 0.89 & 0.96 & 0.89 & 0.89 \\
         & Set Size & 9.93(0.11) & \graycell \textbf{8.84(0.10)} & 10.81(0.14) & 13.48(0.19) & 45.75(0.16) & 1963.68(0.03) & 13.00(0.17) \\
         \midrule
         \multirow{3}{*}{meps19} & Marg. C & 0.90 & \graycell 0.90 & 0.90 & 0.90 & 0.93 & 0.90 & 0.90 \\
         & Cond. C & 0.87 & \graycell 0.88 & 0.89 & 0.89 & 0.92 & 0.88 & 0.89 \\
         & Set Size & 19.28(0.16) & \graycell \textbf{17.78(0.18)} & 18.26(0.15) & 29.96(0.28) & 23.86(0.17) & 559.23(0.01) & 28.71(0.18) \\
         \midrule
         \multirow{3}{*}{meps20} & Marg. C & 0.90 &  0.90 & \graycell 0.90 & 0.90 & 0.92 & 0.90 & 0.90 \\
         & Cond. C & 0.87 &  0.88 & \graycell 0.90 & 0.90 & 0.92 & 0.88 & 0.90 \\
         & Set Size & 19.52(0.16) &  18.19(0.17) & \graycell \textbf{17.94(0.18)} & 29.35(0.23) & 22.93(0.16) & 520.25(0.01) & 27.57(0.15) \\
         \midrule
         \multirow{3}{*}{meps21} & Marg. C & 0.90 & \graycell 0.90 & 0.90 & 0.90 & 0.93 & 0.90 & 0.90 \\
         & Cond. C & 0.87 & \graycell 0.88 & 0.90 & 0.89 & 0.92 & 0.88 & 0.89 \\
         & Set Size & 19.18(0.12) & \graycell \textbf{17.91(0.15)} & 18.65(0.16) & 30.32(0.31) & 23.63(0.17) & 531.25(0.01) & 29.89(0.20) \\
         \midrule
         \multirow{3}{*}{temperature} & Marg. C & 0.90 & \graycell 0.90 & 0.90 & 0.90 & 0.92 & 0.90 & 0.90 \\
         & Cond. C & 0.90 & \graycell 0.89 & 0.89 & 0.89 & 0.91 & 0.88 & 0.87 \\
         & Set Size & 2.10(0.01) & \graycell \textbf{1.85(0.01)} & 3.24(0.01) & 3.07(0.01) & 2.23(0.01) & 3.10(0.02) & 3.55(0.03) \\
         \bottomrule
    \end{tabular}}
    \vspace{1mm}
    \caption{ Summary results of real data experiments, where Marg. C and Cond. C denote the marginal coverage and approximated conditional coverage. The results are averaged over 50 random cross-validation splits.
    We report the set size mean and standard error (inside the parentheses, the default is 0.00) based on the same 50 splits. The nominal coverage rate $(1 - \alpha)$ is 90\%, the $K$ for (HD-)PCP is set as 40. To save space and keep consistency, here we report PCP-MixD, HD-PCP-MixD, CHR-QRF, CDSplit-MixD and CQR, which include the variant that works generally the best across the 9 datasets. 
    The detailed results are placed in \Cref{sec:full_real_exp}. }
    \label{tab:real_data}
\end{table}

\subsection{Multi-dimensional targets}

Beyond single target regression tasks, we  study PCP and HD-PCP on multi-target datasets. %
We use the same strategy in~\citet{neeven2018conformal} to adapt previous baselines to multi-target  conformal algorithms by fitting each dimension separately with coverage level $(1-{\alpha})/{d}$, where $d$ is the dimension of target $Y$ (this ensures the coverage of the target vector is $1-\alpha$ \citep{Lei2020-px}). 

We construct a synthetic dataset to illustrate the benefit of PCP when targets are dependent. Covariates $X$ are sampled from $\mathcal{N}(0,1)^5$ and the target $Y$ is randomly sampled from a bi-modal Gaussian distribution $0.5\mathcal{N}(\mu_1, \Sigma)+0.5\mathcal{N}(\mu_2, \Sigma)$. $\mu_i = \beta_i^T (x,1)$ and $\beta_i$ is randomly generated from $\mathcal{N}(0,1)^6$, $\Sigma = \begin{pmatrix}
10 & \rho\\
\rho & 10
\end{pmatrix}$. %
The synthetic data distribution in Figure~\ref{fig:mdsyndata} shows that the distribution concentrates as $\rho$ increases. As shown in~\Cref{fig:mdsyndatacover} and \Cref{fig:mdsyndataarea},  PCP achieves the best performance in terms of average predictive set size. We observe when $\rho$ gets higher, only PCP shrinks the predictive set accordingly while the predictions from other methods have little change and become loose. The detailed quantitative results are in~\Cref{tab:md_syn_data}, \Cref{appendix:synthetic}.

 \begin{figure}[ht]
     \centering
     \begin{subfigure}[b]{0.52\textwidth}
         \includegraphics[width=\textwidth]{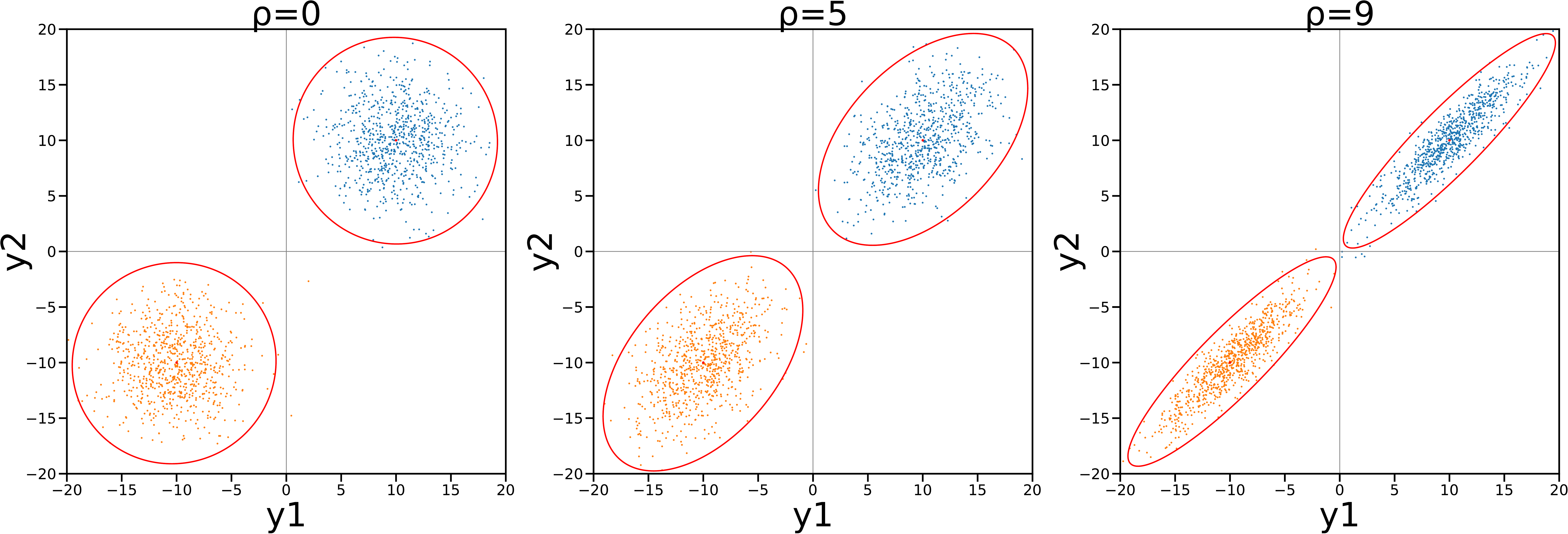}
     \caption{$p(Y|X)$ with varying $\rho$}
     \label{fig:mdsyndata}
     \end{subfigure} \hfill 
     \begin{subfigure}[b]{0.23\textwidth}
         \includegraphics[width=\textwidth]{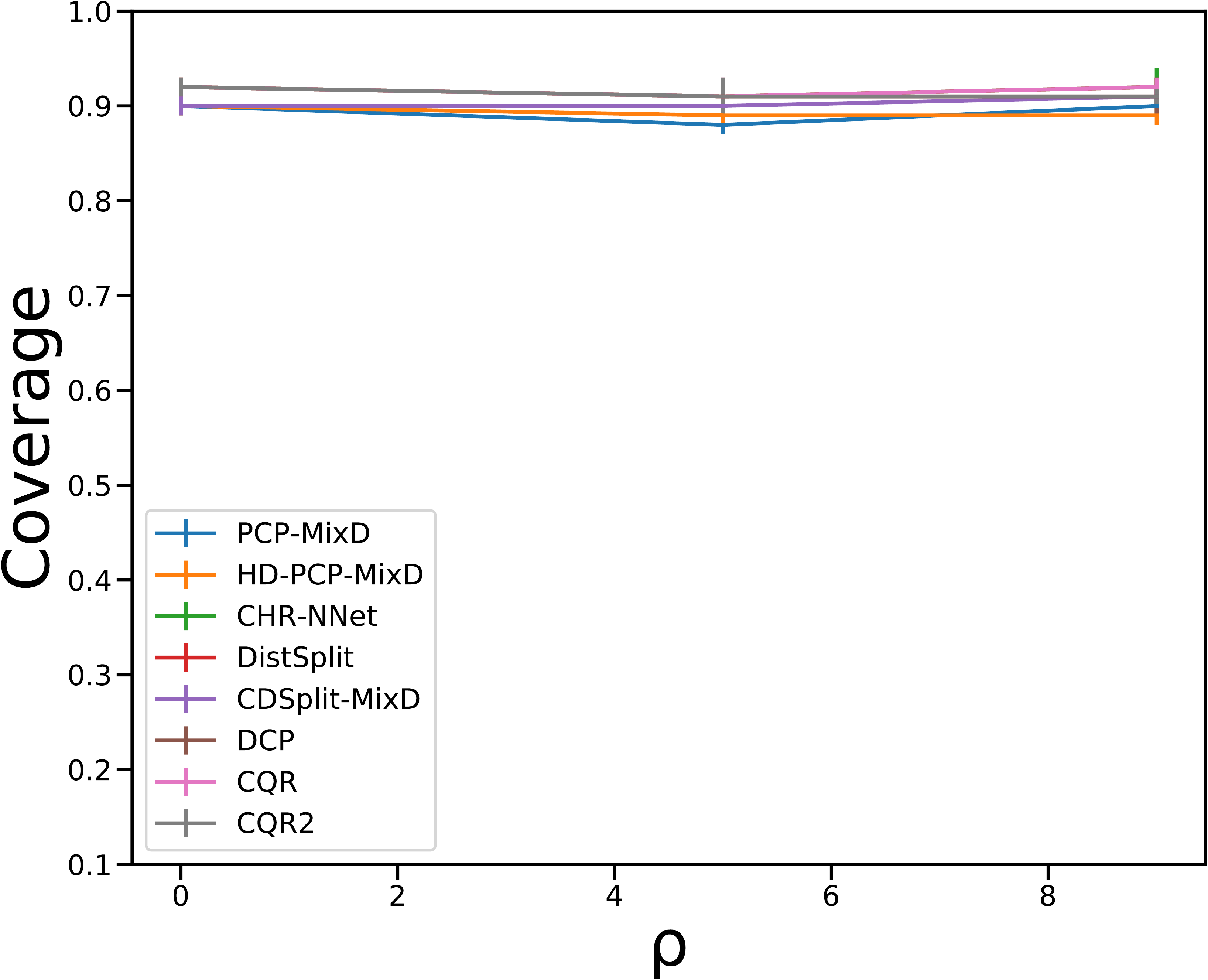}
         \caption{Coverage}
         \label{fig:mdsyndatacover}
     \end{subfigure}\hfill
     \begin{subfigure}[b]{0.23\textwidth}
         \includegraphics[width=\textwidth]{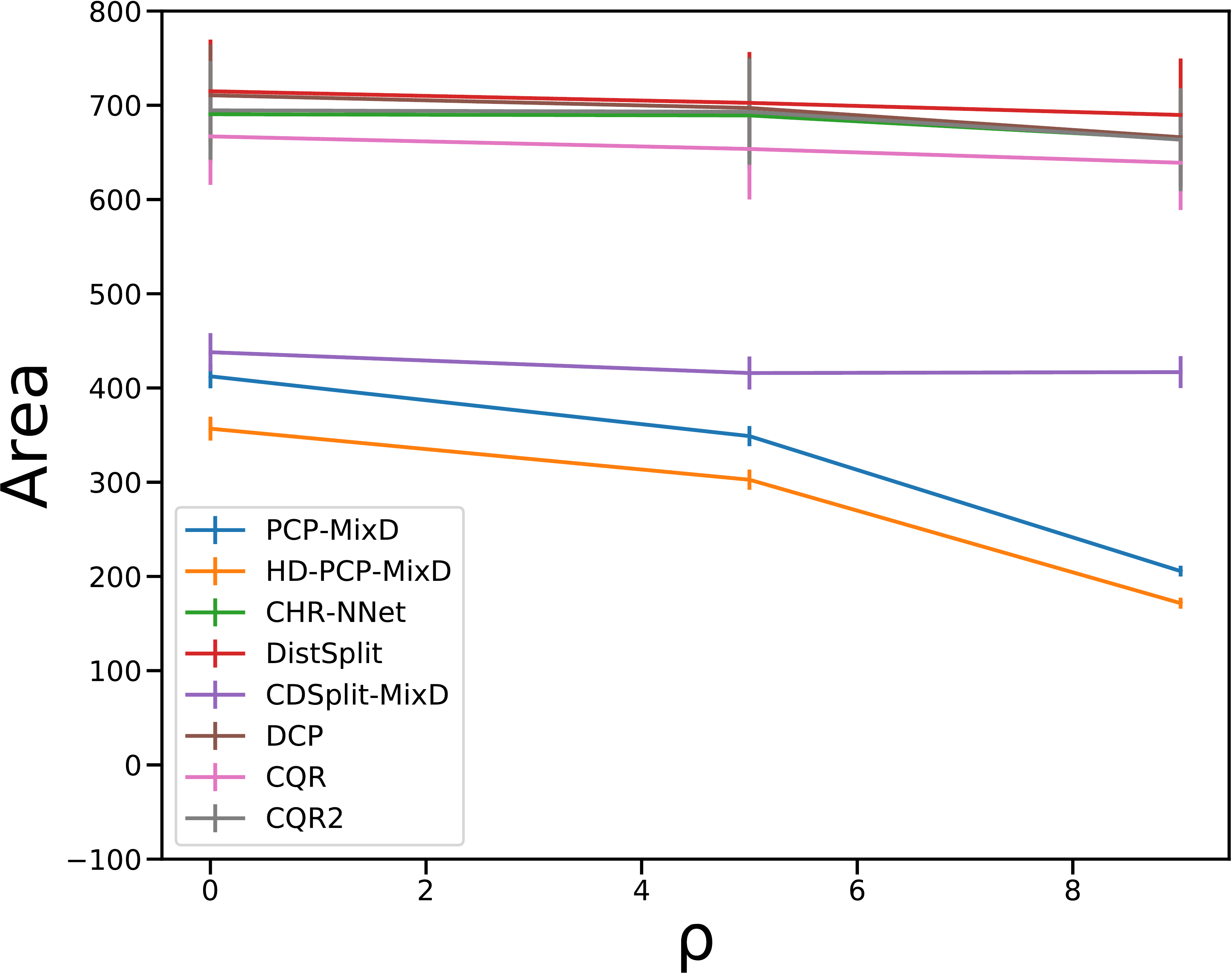}
         \caption{Set Size}
         \label{fig:mdsyndataarea}
     \end{subfigure}   
    \caption{(a): Conditional data distribution $p(Y|X)$ for multi-target synthetic dataset. (b), (c): Marginal coverage and set size for baselines. PCP is able to capture the joint coverage more efficiently than other baselines when $\rho$ increases. %
    }     
 \end{figure}

We further study conformal methods on two multi-target real datasets. %
Taxi Data are the taxi trip records of New York City which include the pickup, drop-off locations of each trip and the corresponding time. We sample 10000 records from January 2016 (7800, 2000, 200 for train, calibration and test respectively), and use the pickup time and location as covariates $X$ to predict drop-off locations $Y$. The locations are represented by latitudes and longitudes.%
For energy dataset, we predict the heating load and cooling load for energy efficiency analysis \citep{tsanas2012accurate}. The predictors are building information such as orientation, glazing area and wall area. We use 568, 100, and 100 samples for train, calibration and test set. We report the predictive set sizes of each algorithm. 
To calculate the predictive set size of PCP, due to the overlapped regions, the set size cannot be calculated exactly. We estimate it by Monte Carlo simulation with a grid size of 100 on each dimension. For (HD-)PCP and CDSplit, we use MixD as the backbone model. ~\looseness=-1

As shown in Table~\ref{tab:md_real_data}, algorithms considering multimodality have significantly better performance compared to other baselines. We visualize the conformal region predicted by (HD-)PCP in \Cref{fig:taxiaddtion} (See \Cref{sec:taxiapp} for more qualitative results). We notice that PCP can successfully capture the most popular regions of New York City for drop-off, downtown of Manhattan, LaGuardia airport and JFK airport, while methods with continuous predictive sets would learn a wide bounding box. Furthermore, as shown in \Cref{tab:md_real_data}, PCP has similar coverage level but with much smaller predictive set compared to CDSplit. This might be because the joint estimation of the high-dimensional targets can capture dependencies between the target elements. Applying the filtering, HD-PCP provides a cleaner and interpretable predictive set and further reduces the set size.  

We include more experiments with two other multi-target real datasets in \Cref{sec:addmultiapp}( 8 and 3 targets) due to space constraint. Since Monte-Carlo estimation of volume of high-dimensional nonconvex set suffers from curse of dimension, we focus on two dimensional task and evaluate the predictive set generated for every pair of all targets. PCP and HD-PCP offer significantly sharper predictive sets compared to baselines on almost all pairs, and HD-PCP can further improve the performance of PCP.

\begin{table}
    \centering
    \resizebox{\textwidth}{!}{
    \begin{tabular}{l|l|c|c|c|c|c|c|c}
         \toprule
         Data & Metric & PCP & HD-PCP & CHR \citep{Sesia2021-ei} & DistSplit \citep{izbicki2020flexible} & CDSplit \citep{izbicki2020flexible} & DCP \citep{Chernozhukov2021DistributionalCP} & CQR \citep{romano2019conformalized} \\
         \midrule
         \multirow{3}{*}{Taxi} 
         & Marg. C & 0.90(0.01) & \graycell 0.89(0.01) & 0.87(0.01) & 0.91(0.01) & 0.91(0.01) & 0.89(0.01) & 0.90(0.01) \\
         & Cond. C & 0.89(0.03) & \graycell 0.87(0.02) & 0.84(0.04) & 0.92(0.03) & 0.89(0.03) & 0.84(0.03) & 0.86(0.04) \\
         & Set Size & 0.0089(0.0001) & \graycell \textbf{0.0064(0.0002)} & 0.0245(0.0009) & 0.0202(0.0009) & 0.0097(0.0023) & 0.0354(0.0013) & 1.9302(0.5202) \\
         \midrule
         \multirow{3}{*}{Energy} 
         & Marg. C & 0.89(0.01) & \graycell 0.90(0.01) & 0.93(0.01) & 0.92(0.01) & 0.93(0.01) & 0.93(0.01) & 0.91(0.01) \\
         & Cond. C & 0.87(0.04) & \graycell 0.94(0.03) & 0.93(0.05) & 0.83(0.10) & 0.89(0.06) & 0.97(0.02) & 0.94(0.03) \\
         & Set Size & 19.22(2.55) & \graycell \textbf{14.13(3.55)} & 27.41(2.65) & 36.31(3.04) & 36.88(3.94) & 34.18(3.61) & 45.20(3.19) \\
         \bottomrule
    \end{tabular}}
    \vspace{1mm}
    \caption{ Summary results of Multi-Target Regression experiments, where Marg. C and Cond. C denote the marginal coverage and approximated conditional coverage. Results are averaged over 10 random cross-validation splits. We report the standard error inside the parentheses. The nominal coverage rate $(1 - \alpha)$ is 90\%, the $K$ for (HD-)PCP is set as 1000.}
    \label{tab:md_real_data}
\end{table}

\section{Conclusion and Future Work} \label{sec:conclusion}
We proposed novel conformal inference algorithms, PCP and HD-PCP, that find valid and sharp predictive  sets  using random  samples from a  conditional generative model; PCP and HD-PCP can be built from either implicit or explicit models.  PCP and HD-PCP outperform existing methods, particularly with multimodal data and multi-dimensional targets.

Like most existing conformal methods, PCP and HD-PCP rely on a reasonable estimation through a fitted conditional generative model. When the generative model is inaccurate, PCP may only provide wide, uninformative predictive sets to ensure the validity. However, given the recent success in conditional generative models, PCP and HD-PCP may be promising in many domains, including precision medicine, quantitative finance and marketing \citep{koshiyama2021generative,lee2018diagnosis,xu2014path}.

Here we focused on regression tasks. Future research can consider adapting PCP and HD-PCP to conformal treatment effect estimation and classification problems~\citep{gao2021enhancing,Romano2020ClassificationWV,Yin2021-fw}. 
\FloatBarrier

\bibliography{arxiv}

\begin{thebibliography}{44}
\providecommand{\natexlab}[1]{#1}
\providecommand{\url}[1]{\texttt{#1}}
\expandafter\ifx\csname urlstyle\endcsname\relax
  \providecommand{\doi}[1]{doi: #1}\else
  \providecommand{\doi}{doi: \begingroup \urlstyle{rm}\Url}\fi

\bibitem[Alsing et~al.(2019)Alsing, Charnock, Feeney, and
  Wandelt]{alsing2019fast}
Justin Alsing, Tom Charnock, Stephen Feeney, and Benjamin Wandelt.
\newblock Fast likelihood-free cosmology with neural density estimators and
  active learning.
\newblock \emph{Monthly Notices of the Royal Astronomical Society},
  488\penalty0 (3):\penalty0 4440--4458, 2019.

\bibitem[Ambrogioni et~al.(2017)Ambrogioni, G{\"u}{\c{c}}l{\"u}, van Gerven,
  and Maris]{Ambrogioni2017}
Luca Ambrogioni, Umut G{\"u}{\c{c}}l{\"u}, Marcel~AJ van Gerven, and Eric
  Maris.
\newblock The kernel mixture network: A nonparametric method for conditional
  density estimation of continuous random variables.
\newblock \emph{arXiv preprint arXiv:1705.07111}, 2017.

\bibitem[Barber et~al.(2019)Barber, Candes, Ramdas, and
  Tibshirani]{Barber2019-rl}
Rina~Foygel Barber, Emmanuel~J Candes, Aaditya Ramdas, and Ryan~J Tibshirani.
\newblock Predictive inference with the jackknife+.
\newblock \emph{arXiv}, May 2019.

\bibitem[Bishop(1994)]{bishop1994mixture}
Christopher~M. Bishop.
\newblock Mixture density networks.
\newblock Technical report, Aston University, 1994.

\bibitem[Breiman and Friedman(1997)]{breiman1997predicting}
Leo Breiman and Jerome~H Friedman.
\newblock Predicting multivariate responses in multiple linear regression.
\newblock \emph{Journal of the Royal Statistical Society: Series B (Statistical
  Methodology)}, 59\penalty0 (1):\penalty0 3--54, 1997.

\bibitem[Cauchois et~al.(2021)Cauchois, Gupta, and
  Duchi]{Cauchois2021KnowingWY}
Maxime Cauchois, Suyash Gupta, and John~C. Duchi.
\newblock Knowing what you know: valid and validated confidence sets in
  multiclass and multilabel prediction.
\newblock \emph{J. Mach. Learn. Res.}, 22:\penalty0 81:1--81:42, 2021.

\bibitem[Chan et~al.(2018)Chan, Perrone, Spence, Jenkins, Mathieson, and
  Song]{chan2018likelihood}
Jeffrey Chan, Valerio Perrone, Jeffrey Spence, Paul Jenkins, Sara Mathieson,
  and Yun Song.
\newblock A likelihood-free inference framework for population genetic data
  using exchangeable neural networks.
\newblock \emph{Advances in neural information processing systems}, 31, 2018.

\bibitem[Chernozhukov et~al.(2021)Chernozhukov, W{\"u}thrich, and
  Zhu]{Chernozhukov2021DistributionalCP}
Victor Chernozhukov, Kaspar W{\"u}thrich, and Yinchu Zhu.
\newblock Distributional conformal prediction.
\newblock \emph{Proceedings of the National Academy of Sciences}, 118, 2021.

\bibitem[Cho et~al.(2020)Cho, Yoo, Im, and Cha]{Cho2020ComparativeAO}
Dongjin Cho, Cheolhee Yoo, Jungho Im, and Dong‐Hyun Cha.
\newblock Comparative assessment of various machine learning‐based bias
  correction methods for numerical weather prediction model forecasts of
  extreme air temperatures in urban areas.
\newblock \emph{Earth and Space Science}, 7, 2020.

\bibitem[Dorn and Guo(2022)]{dorn2022sharp}
Jacob Dorn and Kevin Guo.
\newblock Sharp sensitivity analysis for inverse propensity weighting via
  quantile balancing.
\newblock \emph{Journal of the American Statistical Association}, pages 1--28,
  2022.

\bibitem[Gao et~al.(2021)Gao, Biggs, Sun, and Han]{gao2021enhancing}
Ruijiang Gao, Max Biggs, Wei Sun, and Ligong Han.
\newblock Enhancing counterfactual classification via self-training.
\newblock \emph{arXiv preprint arXiv:2112.04461}, 2021.

\bibitem[Geisser(1993)]{geisser1993predictive}
Seymour Geisser.
\newblock \emph{Predictive Inference}, volume~55.
\newblock CRC Press, 1993.

\bibitem[Gillmann et~al.(2021)Gillmann, Peter, Schmidt, Saur, and
  Scheuermann]{gillmann2021visualizing}
Christina Gillmann, Lucas Peter, Carlo Schmidt, Dorothee Saur, and Gerik
  Scheuermann.
\newblock Visualizing multimodal deep learning for lesion prediction.
\newblock \emph{IEEE Computer Graphics and Applications}, 41\penalty0
  (5):\penalty0 90--98, 2021.

\bibitem[Goodfellow et~al.(2014)Goodfellow, Pouget-Abadie, Mirza, Xu,
  Warde-Farley, Ozair, Courville, and Bengio]{Goodfellow2014GenerativeAN}
Ian~J. Goodfellow, Jean Pouget-Abadie, Mehdi Mirza, Bing Xu, David
  Warde-Farley, Sherjil Ozair, Aaron~C. Courville, and Yoshua Bengio.
\newblock Generative adversarial nets.
\newblock In \emph{NIPS}, 2014.

\bibitem[Hoff(2021)]{Hoff2021-hk}
Peter Hoff.
\newblock Bayes-optimal prediction with frequentist coverage control.
\newblock \emph{arXiv}, 2021.

\bibitem[Izbicki et~al.(2020)Izbicki, Shimizu, and Stern]{izbicki2020flexible}
Rafael Izbicki, Gilson Shimizu, and Rafael Stern.
\newblock Flexible distribution-free conditional predictive bands using density
  estimators.
\newblock In \emph{International Conference on Artificial Intelligence and
  Statistics}, pages 3068--3077. PMLR, 2020.

\bibitem[Koshiyama et~al.(2021)Koshiyama, Firoozye, and
  Treleaven]{koshiyama2021generative}
Adriano Koshiyama, Nick Firoozye, and Philip Treleaven.
\newblock Generative adversarial networks for financial trading strategies
  fine-tuning and combination.
\newblock \emph{Quantitative Finance}, 21\penalty0 (5):\penalty0 797--813,
  2021.

\bibitem[Kuchibhotla(2020)]{Kuchibhotla2020-se}
Arun~Kumar Kuchibhotla.
\newblock Exchangeability, conformal prediction, and rank tests.
\newblock \emph{arXiv}, May 2020.

\bibitem[Lee et~al.(2018)Lee, Park, Joo, and Moon]{lee2018diagnosis}
Wonsung Lee, Sungrae Park, Weonyoung Joo, and Il-Chul Moon.
\newblock Diagnosis prediction via medical context attention networks using
  deep generative modeling.
\newblock In \emph{2018 IEEE International Conference on Data Mining (ICDM)},
  pages 1104--1109. IEEE, 2018.

\bibitem[Lei and Wasserman(2014)]{lei2014distribution}
Jing Lei and Larry Wasserman.
\newblock Distribution-free prediction bands for non-parametric regression.
\newblock \emph{Journal of the Royal Statistical Society: Series B: Statistical
  Methodology}, pages 71--96, 2014.

\bibitem[Lei et~al.(2015)Lei, Rinaldo, and Wasserman]{lei2015conformal}
Jing Lei, Alessandro Rinaldo, and Larry Wasserman.
\newblock A conformal prediction approach to explore functional data.
\newblock \emph{Annals of Mathematics and Artificial Intelligence}, 74\penalty0
  (1):\penalty0 29--43, 2015.

\bibitem[Lei et~al.(2018)Lei, G’Sell, Rinaldo, Tibshirani, and
  Wasserman]{lei2018distribution}
Jing Lei, Max G’Sell, Alessandro Rinaldo, Ryan~J Tibshirani, and Larry
  Wasserman.
\newblock Distribution-free predictive inference for regression.
\newblock \emph{Journal of the American Statistical Association}, 113\penalty0
  (523):\penalty0 1094--1111, 2018.

\bibitem[Lei and Cand{\`e}s(2021)]{Lei2020-px}
Lihua Lei and Emmanuel~J Cand{\`e}s.
\newblock Conformal inference of counterfactuals and individual treatment
  effects.
\newblock \emph{Journal of the Royal Statistical Society: Series B}, 2021.

\bibitem[Meinshausen(2006)]{Meinshausen2006QuantileRF}
Nicolai Meinshausen.
\newblock Quantile regression forests.
\newblock \emph{J. Mach. Learn. Res.}, 7:\penalty0 983--999, 2006.

\bibitem[Messoudi et~al.(2020)Messoudi, Destercke, and
  Rousseau]{messoudi2020conformal}
Soundouss Messoudi, S{\'e}bastien Destercke, and Sylvain Rousseau.
\newblock Conformal multi-target regression using neural networks.
\newblock In \emph{Conformal and Probabilistic Prediction and Applications},
  pages 65--83. PMLR, 2020.

\bibitem[Messoudi et~al.(2021)Messoudi, Destercke, and
  Rousseau]{messoudi2021copula}
Soundouss Messoudi, S{\'e}bastien Destercke, and Sylvain Rousseau.
\newblock Copula-based conformal prediction for multi-target regression.
\newblock \emph{Pattern Recognition}, 120:\penalty0 108101, 2021.

\bibitem[Mirza and Osindero(2014)]{mirza2014conditional}
Mehdi Mirza and Simon Osindero.
\newblock Conditional generative adversarial nets.
\newblock \emph{arXiv preprint arXiv:1411.1784}, 2014.

\bibitem[Myers et~al.(2022)Myers, Biyik, Anari, and Sadigh]{myers2022learning}
Vivek Myers, Erdem Biyik, Nima Anari, and Dorsa Sadigh.
\newblock Learning multimodal rewards from rankings.
\newblock In \emph{Conference on Robot Learning}, pages 342--352. PMLR, 2022.

\bibitem[Neeven and Smirnov(2018)]{neeven2018conformal}
Jelmer Neeven and Evgueni Smirnov.
\newblock Conformal stacked weather forecasting.
\newblock In \emph{Conformal and Probabilistic Prediction and Applications},
  pages 220--233. PMLR, 2018.

\bibitem[Papadopoulos et~al.(2002)Papadopoulos, Proedrou, Vovk, and
  Gammerman]{papadopoulos2002inductive}
Harris Papadopoulos, Kostas Proedrou, Volodya Vovk, and Alex Gammerman.
\newblock Inductive confidence machines for regression.
\newblock In \emph{European Conference on Machine Learning}. Springer, 2002.

\bibitem[Romano et~al.(2019)Romano, Patterson, and
  Candes]{romano2019conformalized}
Yaniv Romano, Evan Patterson, and Emmanuel Candes.
\newblock Conformalized quantile regression.
\newblock \emph{Advances in neural information processing systems}, 32, 2019.

\bibitem[Romano et~al.(2020)Romano, Sesia, and
  Cand{\`e}s]{Romano2020ClassificationWV}
Yaniv Romano, Matteo Sesia, and Emmanuel~J. Cand{\`e}s.
\newblock Classification with valid and adaptive coverage.
\newblock \emph{arXiv: Methodology}, 2020.

\bibitem[Sesia and Romano(2021)]{Sesia2021-ei}
Matteo Sesia and Yaniv Romano.
\newblock Conformal prediction using conditional histograms.
\newblock \emph{Advances in Neural Information Processing Systems}, 34, 2021.

\bibitem[Shafer and Vovk(2008)]{shafer2008tutorial}
Glenn Shafer and Vladimir Vovk.
\newblock A tutorial on conformal prediction.
\newblock \emph{Journal of Machine Learning Research}, 9\penalty0 (3), 2008.

\bibitem[Spyromitros-Xioufis et~al.(2016)Spyromitros-Xioufis, Tsoumakas,
  Groves, and Vlahavas]{spyromitros2016multi}
Eleftherios Spyromitros-Xioufis, Grigorios Tsoumakas, William Groves, and
  Ioannis Vlahavas.
\newblock Multi-target regression via input space expansion: treating targets
  as inputs.
\newblock \emph{Machine Learning}, 104\penalty0 (1):\penalty0 55--98, 2016.

\bibitem[Tibshirani et~al.(2019)Tibshirani, Foygel~Barber, Candes, and
  Ramdas]{tib2019conformal}
Ryan~J Tibshirani, Rina Foygel~Barber, Emmanuel Candes, and Aaditya Ramdas.
\newblock Conformal prediction under covariate shift.
\newblock In \emph{Advances in Neural Information Processing Systems},
  volume~32. Curran Associates, Inc., 2019.

\bibitem[Trippe and Turner(2018)]{Trippe2018-vp}
Brian~L Trippe and Richard~E Turner.
\newblock Conditional density estimation with bayesian normalising flows.
\newblock \emph{arXiv preprint arXiv:1802.04908}, 2018.

\bibitem[Tsanas and Xifara(2012)]{tsanas2012accurate}
Athanasios Tsanas and Angeliki Xifara.
\newblock Accurate quantitative estimation of energy performance of residential
  buildings using statistical machine learning tools.
\newblock \emph{Energy and buildings}, 49, 2012.

\bibitem[Vovk et~al.(2005)Vovk, Gammerman, and Shafer]{vovk2005algorithmic}
Vladimir Vovk, Alex Gammerman, and Glenn Shafer.
\newblock \emph{Algorithmic learning in a random world}.
\newblock Springer Science \& Business Media, 2005.

\bibitem[Wang and Zhou(2020)]{wang2020thompson}
Zhendong Wang and Mingyuan Zhou.
\newblock Thompson sampling via local uncertainty.
\newblock In \emph{International Conference on Machine Learning}, pages
  10115--10125. PMLR, 2020.

\bibitem[Xu et~al.(2014)Xu, Duan, and Whinston]{xu2014path}
Lizhen Xu, Jason~A Duan, and Andrew Whinston.
\newblock Path to purchase: A mutually exciting point process model for online
  advertising and conversion.
\newblock \emph{Management Science}, 2014.

\bibitem[Yin and Zhou(2018{\natexlab{a}})]{Yin2018SemiImplicitVI}
Mingzhang Yin and Mingyuan Zhou.
\newblock Semi-implicit variational inference.
\newblock In \emph{International Conference on Machine Learning}, pages
  5660--5669. PMLR, 2018{\natexlab{a}}.

\bibitem[Yin and Zhou(2018{\natexlab{b}})]{yin2019semi}
Mingzhang Yin and Mingyuan Zhou.
\newblock Semi-implicit generative model.
\newblock \emph{Proceedings of the NeurIPS Workshop on Bayesian Deep Learning},
  2018{\natexlab{b}}.

\bibitem[Yin et~al.(2021)Yin, Shi, Wang, and Blei]{Yin2021-fw}
Mingzhang Yin, Claudia Shi, Yixin Wang, and David~M. Blei.
\newblock Conformal sensitivity analysis for individual treatment effects.
\newblock \emph{arXiv}, 2021.

\end{thebibliography}
\bibliographystyle{plainnat}

\newpage
\appendix

\begin{center}{\Large{\textbf{Appendix}}}\end{center}

\section{Proof} \label{sec:proof}
For the completeness, we first present a lemma adapted from \citet{tib2019conformal,romano2019conformalized}. 
\begin{lemma} \label{lem:qt}
Suppose $Z_1, \cdots, Z_n, Z_{n+1}$ are  scalar random variables that are exchangeable and almost surely distinct, then for $\beta \in [0,1]$
\bas{
 \beta \leq \mathbb{P}(Z_{n+1} \leq Q_{\beta}(Z_{1:n} \cup \{\infty\}))  \leq  \beta + \frac{1}{n+1}. 
}
\end{lemma}
\begin{proof}[Proof of \Cref{lem:qt}]
Denote the inflated $Z_{1:n}$ as
\ba{
\tilde{Z}_{1:n} =  Z_{1:n} \cup \{\infty\}.
\label{eq:lm0}
}

The quantile 
$$Q_{\beta}(Z_{1:n}) \vcentcolon= \inf_{x} (\sum_{i=1}^n \mathbbm{1}[Z_i \leq x])/n = Z_{(\ceil*{n\beta})}.$$ 

By Lemma 1 of \citet{tib2019conformal}, the events 
\ba{
Z_{n+1} \leq Q_{\beta}(\tilde{Z}_{1:n}) \Leftrightarrow  Z_{n+1} \leq Q_{\beta}(Z_{1:n+1}).  
\label{eq:lm1}
}
Furthermore, by exchangeability and the definition of empirical quantile, 
\ba{
 \beta \leq \mathbb{P}(Z_{n+1} \leq Q_{\beta}(Z_{1:n+1}))  \leq  \beta + \frac{1}{n+1},
 \label{eq:lm2}
}
where the upper bound hold when $Z_{1:n+1}$ are almost surely distinct. By \Cref{eq:lm1}, the proof is completed. 
\end{proof}

\begin{proof}[Proof of \Cref{thm:cov}] Given the estimated conditional density $q(Y|X)$ on an independently sampled training set $D_{\text{tr}}$. By assumption, the calibration set  $\{(X_i,Y_i)\}_{i=1}^n$ and the test point $(X_{n+1}, Y_{n+1})$ are exchangeable. Denote $D_i = (X_i,Y_i, \hat{\Ymat}_i)$ for $i=1,\cdots, n, n+1$. Then  $D_i \sim p(X,Y)q^K(Y | X_i)$ and $\{D_{1:n+1}\}$ are exchangeable because $\hat{\Ymat}_i \indep \hat{\Ymat}_j$. 

The nonconformity score $E_i$  in \Cref{eq:score} is defined as a deterministic function of $D_i $. Therefore $\{E_i\}_{i=1}^{n+1}$ are exchangeable \citep{Kuchibhotla2020-se} and are almost surely distinct. By \Cref{lem:qt}, 
\ba{
1- \alpha \leq \mathbb{P}(E_{n+1} \leq Q_{\alpha}(E_{1:n} \cup \{\infty\}) \g D_{\text{tr}})  \leq  1- \alpha + \frac{1}{n+1}. 
\label{eq:thm1}
}
Next we demonstrate that for $$\hat{C}(X_{n+1}, \hat{\Ymat}_{n+1}) =  \cup_{k=1}^K \big\{y: \norm{y - \hat{Y}_{n+1,k}}\leq \hat{Q}\big\}, ~\hat{Q} = Q_{1-\alpha}(E_{1:n}\cup \{\infty\}), $$
the following statement holds
\bas{
Y_{n+1} \in \hat{C}(X_{n+1}, \hat{\Ymat}_{n+1})   \Leftrightarrow E_{n+1} \leq \hat{Q}.
} 
Suppose the LHS is true, then $\exists$ $m, 1 \leq m \leq K$,  s.t. $Y_{n+1} \in \big\{y: \norm{y - \hat{Y}_{n+1,m}}\leq \hat{Q}\big\}$. This means $\norm{Y_{n+1} - \hat{Y}_{n+1,m}} \leq \hat{Q}$. Hence $E_{n+1} = \min_k \norm{Y_{n+1} - \hat{Y}_{n+1,k}} \leq  \norm{Y_{n+1} - \hat{Y}_{n+1,m}} \leq  \hat{Q}$.

On the other hand, suppose the RHS is true, letting $t = \argmin_k \norm{Y_{n+1} - \hat{Y}_{n+1,k}}$, we have $\norm{Y_{n+1} - \hat{Y}_{n+1,t}} \leq \hat{Q}$, i.e., $Y_{n+1} \in \big\{y: \norm{y - \hat{Y}_{n+1,t}}\leq \hat{Q}\big\}$. Therefore, $Y_{n+1} \in \hat{C}(X_{n+1}, \hat{\Ymat}_{n+1})$.

Then by \Cref{eq:thm1}, we have 
\ba{
1- \alpha \leq \mathbb{P}(Y_{n+1} \in \hat{C}(X_{n+1}, \hat{\Ymat}_{n+1})   \g D_{\text{tr}})  \leq  1- \alpha + \frac{1}{n+1}. 
}
Marginalizing out $D_{\text{tr}}$  the statement is proved. 
\end{proof}

\begin{proof}[Proof of \Cref{cor:cov}] 
Suppose $0 \leq \beta \leq n/(n+1)$, then $\ceil*{(n+1)\beta} \neq n+1$.  We have 
\ba{
Q_{\beta}(\tilde{Z}_{1:n}) = \tilde{Z}_{(\ceil*{(n+1)\beta}, n+1)} = Z_{(\ceil*{n \frac{n+1}{n} \beta, n})} = Q_{(1+\frac1n)\beta}(Z_{1:n}) 
}
where $\tilde{Z}_{1:n}$ is defined in \Cref{eq:lm0}. By \Cref{eq:lm1,eq:lm2}, 
\ba{
 \beta \leq \mathbb{P}(Z_{n+1} \leq Q_{(1+\frac1n)\beta}(Z_{1:n}))  \leq  \beta + \frac{1}{n+1}. 
 \label{eq:cor1}
}
By \Cref{eq:cor1}, we have
\ba{
 \beta - \frac{1}{n+1} \leq \mathbb{P}(Z_{n+1} \leq Q_{\beta}(Z_{1:n}))  \leq  \beta + \frac{1}{n+1}. 
}
\end{proof}

\begin{proof}[Proof of \Cref{cor:hdpcp}]
With the notations in the proof of \Cref{thm:cov}, $D_i = (X_i,Y_i, \hat{\Ymat}_i)$ are i.i.d. varariables. For a fixed conditional density function $q(y|x)$, the nonconformity score $E_i$ is fully determined by $X_i, Y_i, \hat{\Ymat}_i$. So $E_i = g(D_i)$ where $g(\cdot)$ is a deterministic function including the filtering step of HD-PCP. By \citet{Kuchibhotla2020-se}, since $\{D_i\}_{i=1}^{n+1}$ are i.i.d., $\{E_i\}_{i=1}^{n+1}$ are exchangeable. The other parts of proof follow the same as the proof of \Cref{thm:cov}.
\end{proof}

\section{High Density Probabilistic Conformal Prediction}
\label{sec:hdpcpfigure}
Figure~\ref{fig:hdpcp} shows the different predictive sets with $95\%$ coverage when the  underlying distribution is bi-mode normal. 

We summarize our HD-PCP algorithm in \Cref{alg:hd_pcp}. 

\begin{figure}
    \centering
    \includegraphics[width=0.5\textwidth]{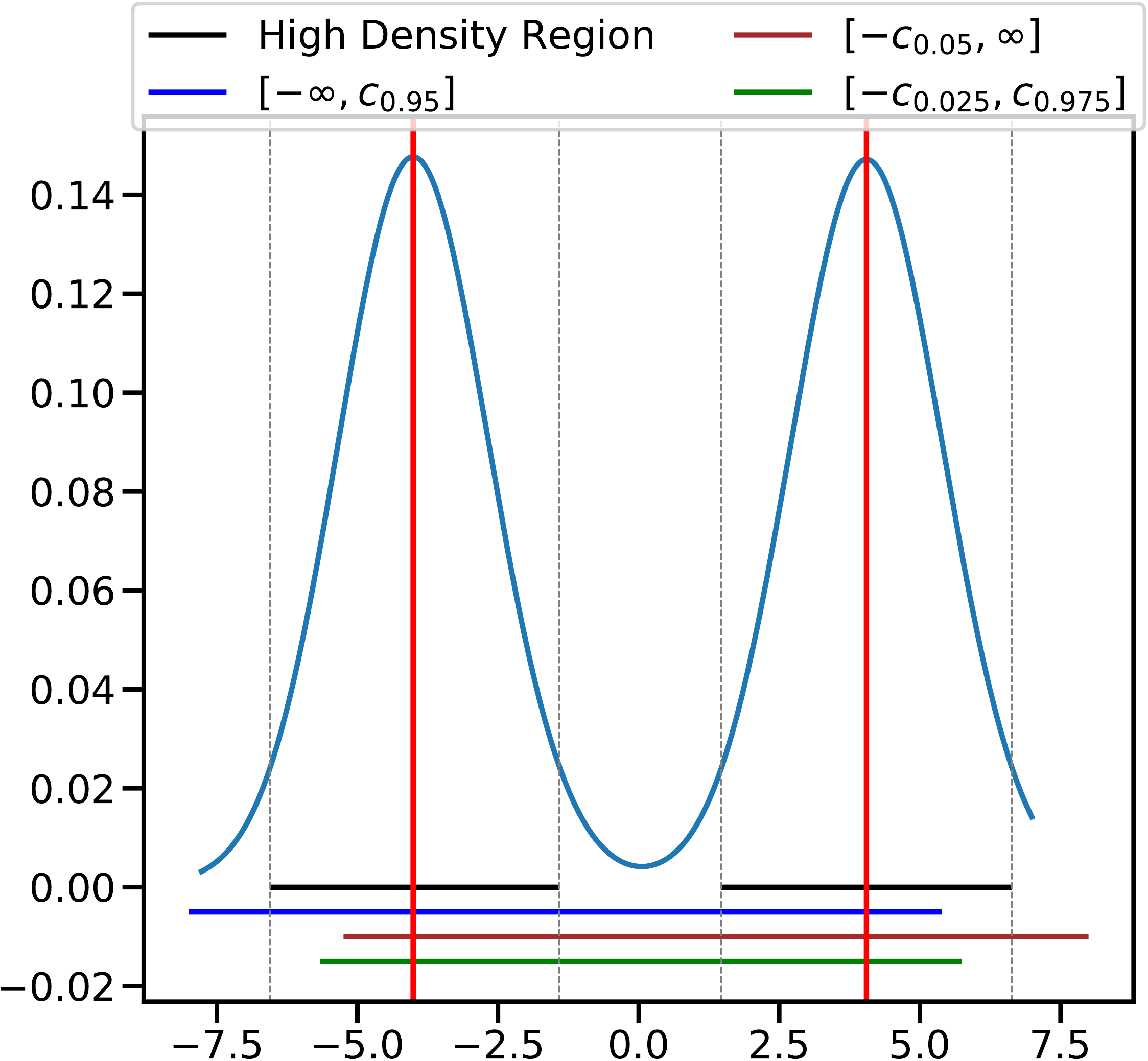}
    \caption{High Density Region can capture multi-mode easier comparing to other predictive set.}
    \label{fig:hdpcp}
\end{figure}

\begin{algorithm}[!t]
  \caption{High Density Probabilistic Conformal Prediction} 
  \label{alg:hd_pcp}
\textbf{Input: } Data $\Dmat = \{(X_{i}, Y_{i})\}_{i=1}^{N}$, nominal level $\alpha$, test point $X$, generative model class $\mathcal{Q}$, sample size $K$, $\beta$ grid $B$ (For HD-PCP). 

  \vspace{0.2cm}
  \textbf{Step I: Conditional generative model}
  \begin{algorithmic}[1]
    \STATE Split the data into three folds $\Z_{\text{tr}}$, $\Z_{\text{val}}$, $\Z_{\text{cal}}$ with set of index as $\bI_{\text{tr}}$, $\bI_{\text{val}}$, $\bI_{\text{cal}}$ respectively
    \STATE Estimate  $q(Y|X)$ on $\Z_{\text{tr}}$  with hyper-parameter chosen by cross validation on $\Z_{\text{val}}$
  \end{algorithmic}
  \vspace{0.2cm}
  \textbf{Step II: Predictive set for a test point}
    \begin{algorithmic}[1]
    \STATE For $i \in \bI_{\text{cal}} $, sample $\hat{Y}_{i1}, \cdots, \hat{Y}_{iK} \sim q(Y|X_i)$ 
    \STATE For $\beta \in B$, Filtering out $\beta$ fraction of $\{\hat{Y}_{ik}\}_{k=1}^K$ with the lowest density. Repeat Line 3-7 for $x\in \bI_{\text{cal}}$. $\beta_0=\argmin_{\beta} \lambda(\sum_{x\in\bI_{\text{cal}}}\hat{C}_{\beta}(x, \hat{\Ymat}))$
    
    \STATE For a test point, sample $\hat{Y}_{1}, \cdots, \hat{Y}_{K} \sim q(Y|X)$
    
    \STATE Filtering out $\beta$ fraction of $\{\hat{Y}_{k}\}_{k=1}^K$ with the lowest density  
    
     \STATE Compute nonconformity score $ \{E_i\}_{i \in {\bI}_{\text{cal}} }$ by \Cref{eq:score}, $E_{N+1}=\infty$, $\tilde{\bI}_{\text{cal}} = \bI_{\text{cal}} \cup \{N+1\}$
     
     \STATE Set $r$ as the $(1-\alpha)$ empirical quantile of $ \{E_i\}_{i \in \tilde{\bI}_{\text{cal}} }$
 
 \STATE Compute the predictive set $\hat{C}_{\beta}(X, \hat{\Ymat})$ by \Cref{eq:interval} 
  \end{algorithmic}
  \textbf{Output:} Predictive set $\hat{C}_{\beta_0}(X, \hat{\Ymat})$ 
\end{algorithm}

\section{Hyperparameters}

PCP introduces one additional hyperparameter, the sample size $K$. We conduct ablation study on the effect of $K$ in \Cref{fig:ablation_K} and find that as long as $K$ is set moderately large, i.e., $K=40$, the predictive set size is near optimal. $K$ is not a very sensitive hyperparameter that needs much effort for tuning. 

\section{Summary of Conditional generative models} \label{sec:appendix_cde}

\textbf{SIVI Model.} Following \citep{wang2020thompson}, we build a conditional distribution estimator by using semi-implicit variational inference \citep{Yin2018SemiImplicitVI} and we call it SIVI model. Specifically, we approximate $\mathbb{P}(Y \g X)$ by an inference distribution $q_{\phiv}(Y\g X)$ with respect to parameter $\phiv$. We construct the inference distribution as a hierarchical model,
$$
y \sim q_{\phiv_1}(y | x, \zv), \zv \sim q_{\phiv_2}(\zv | x, \psiv), \psiv \sim p(\psiv)
$$
Here $\zv \in \bR^d$ is an auxiliary latent variable and $p(\psiv)$ is a known noise distribution, i.e., $\cN (\mathbf{0}, \Imat)$. The inference distribution is the marginal distribution of the
hierarchical model, $q_{\phiv}(y \g x) = \int q_{\phiv_1}(y | \zv)q_{\phiv_2}(\zv | x)d\zv$ with $\phiv = (\phiv_1, \phiv_2)$. We model $q_{\phiv_1}(y | \zv)$ and $q_{\phiv_2}(\zv | x)$ as Gaussian distributions, whose mean and standard deviation are the outputs of neural networks feed with corresponding $(x, \zv)$ and $(x, \psiv)$. The marginal distribution $q(y \g x)$ can thus be constructed with flexibility in modeling multimodality, skewness and kurtosis. We learn the $\phiv$ by maximizing the ELBO \citep{wang2020thompson}, 
$$
\cL_{K} = \mathbb{E}_{\psiv^{(0)}, \dots, \psiv^{(K)} \stackrel{iid} \sim p(\psiv)}
\mathbb{E}_{\zv \sim q(\cdotv \g \psiv^{(0)},\xv)}
\bigg[ \ln q_{\phiv_1}(y \g \xv, \zv) + \ln \frac{p(\zv)}{\frac{1}{K+1}\sum_{k=0}^K q_{\phiv_2}(\zv \g \psiv^{(k)},\xv)} \bigg],
$$
where $p(\zv)$ is a prior distribution for latent variable $\zv$, i.e.,$\cN (\mathbf{0}, \Imat)$ and $K$ is set as 20. 

\textbf{GAN model.} To fit a conditional distribution, we follows \citep{mirza2014conditional} to build a Conditional GAN model with a generator $G(\xv, \zv) $ and a discriminator $D(\xv, y)$. For simplicity, we call it GAN model. Here, the $\zv$ is a latent variable, which is usually set as $\zv \sim \cN (\mathbf{0}, \Imat)$, $G(\xv, \zv)$ is modeled by a neural network, whose outputs are the samples of $y$, and $D(\xv, y)$ is another neural network, which outputs the probability that the given $y$ is from the true data distribution. We train $G$ and $D$ with the following adversarial loss,
$$
\min_{G} \max_{D} V(G, D) = \mathbb{E}_{\xv, y \sim p_{data}(\xv, y)}[\log D(\xv, y)] + \mathbb{E}_{\zv \sim p(\zv), \xv \sim p_{data}(\xv)}[\log (1 - D(\xv, G(\xv, \zv)))]
$$

\textbf{Kernel Mixture Network.} \citet{Ambrogioni2017} model arbitrarily complex conditional densities as linear combinations of a family of kernel functions centered at a subset of training points. The weights are determined by the outer layer of a deep neural network, trained by minimizing the negative log likelihood. The conditional density function is modeled as follows,
$$
q(y |x) = \frac{1}{\sum_{p, j} w_{p, j} (x; W)} \sum_{p, j} w_{p, j} (x; W) \mathcal{K}_j(y, y^{(p)}), 
$$
where $p$ denotes the index of the observed data points, $\mathcal{K}_j$ is the pre-set kernel function, $j$ is the index of selected bandwidth for $\mathcal{K}_j$, and $w_{p, j} (x; W)$ represents the weight of each kernel. A common choice for $\mathcal{K}_j$ is the Gaussian kernel,
$$
\mathcal{K}_j(y, y'; \sigma_j) = \frac{1}{\sqrt{2 \pi} \sigma_j} \exp^{-\frac{(y - y')^2}{2 \sigma_j^2}}.
$$
The weights $w_{p, j} (x; W)$ are determined by a deep neural network (DNN), with covariates $x$ as the inputs and $W$ as the parameters. All weights are non negative by applying non-negative activation functions on the output layer of DNN. We train the KMN model by minimizing the loss function,
$$
\cL(W) = -\sum_q \bigg [ \log \sum_{p, j} w_{p, j} (x; W) \mathcal{K}_j(y, y^{(p)}) - \log \sum_{p, j} w_{p, j} (x; W) \bigg]
$$

\textbf{Mixture Density Network.} \citet{bishop1994mixture} proposes the mixture density network as fellows,
$$
q(y |x) = \sum_{k=1}^K \pi_k(x) \cN (y | \mu_k(x), \sigma_k^2(x))
$$
where $\pi_k(\cdot)$, $\mu_k(\cdot)$ and $\sigma_k(\cdot)$ are all modeled by neural networks. $\sum_{k=1}^K \pi_k(x) = 1$ is guaranteed by using softmax activation function. The model is trained by minimizing the loss function,
$$
\cL = - \sum_{i=1}^N \log \bigg [ \sum_{k=1}^K \pi_k(x_i) \cN (y_i | \mu_k(x_i), \sigma_k^2(x_i)) \bigg],
$$
where $\{(x_i, y_i)\}_{i=1}^N$ are the observed data points. 

\textbf{Quantile Regression Forest} \citet{Meinshausen2006QuantileRF} shows that random forests provide information about the full conditional distribution of the response variable, not only about the conditional mean. Conditional quantiles can be inferred with quantile regression forests, a generalisation of random forests. Quantile regression forests give a non-parametric and accurate way of estimating conditional quantiles for high-dimensional predictor variables. We refer to \citep{Meinshausen2006QuantileRF} for more details about QRF model. PCP needs to get samples from QRF. We first sample a percentile $\tau \sim U[0, 1]$, the uniform distribution on the unit interval, and then use QRF to get the estimated conditional quantile value $y_{\tau}$ as a $y$ sample.

\section{Full synthetic experiment results} \label{appendix:synthetic}

We include the Full synthetic experiment results for 2D toy datasets, s-curve, half-moons, 25-Gaussians, 8-Gaussians, circle and swiss-roll, in \Cref{fig:synthetic_all_fr_0.0} and \Cref{fig:synthetic_all_fr_0.2}. We compare conformal prediction with mean estimation (CP-MeanPred), CHR-QRF and CDSplit-MixD with our method (HD-)PCP. 

\begin{figure}
    \centering
    \includegraphics[width=\textwidth]{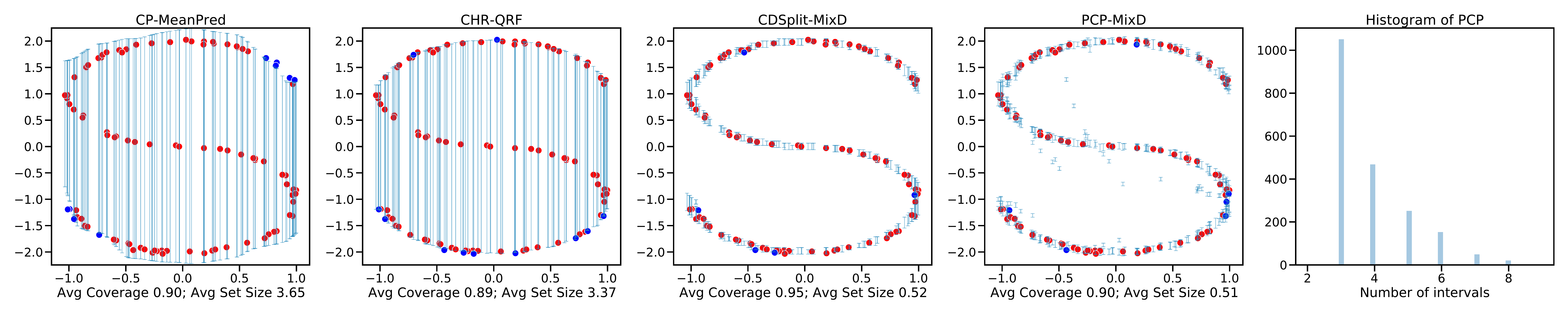} \\
    \includegraphics[width=\textwidth]{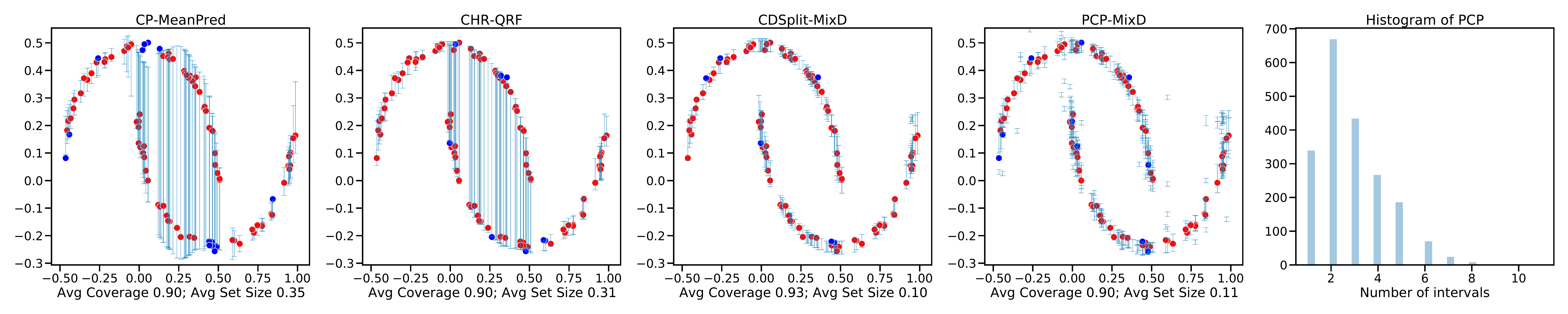} \\
    \includegraphics[width=\textwidth]{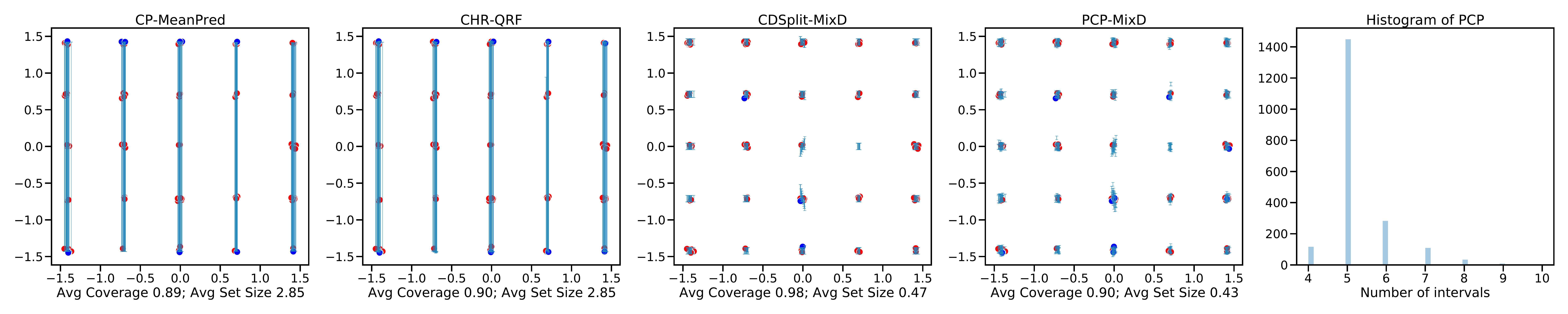} \\
    \includegraphics[width=\textwidth]{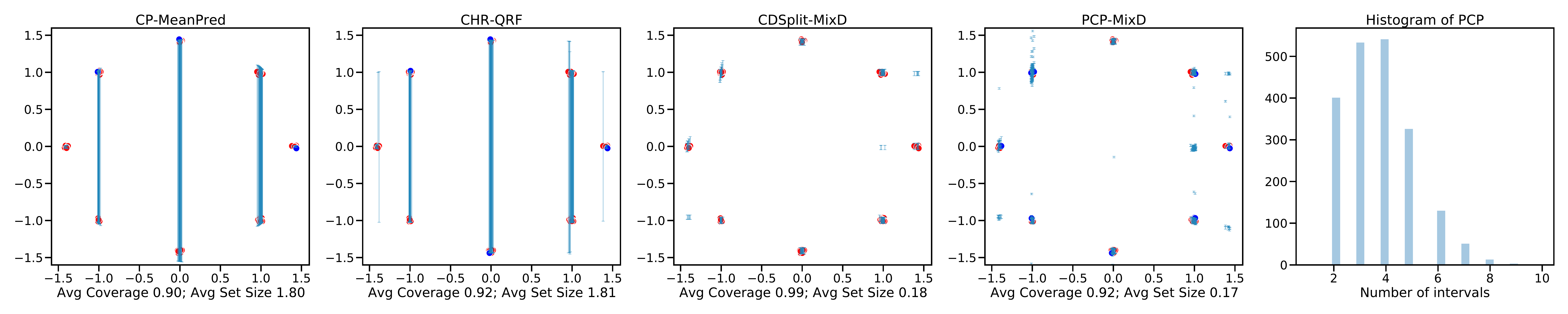} \\
    \includegraphics[width=\textwidth]{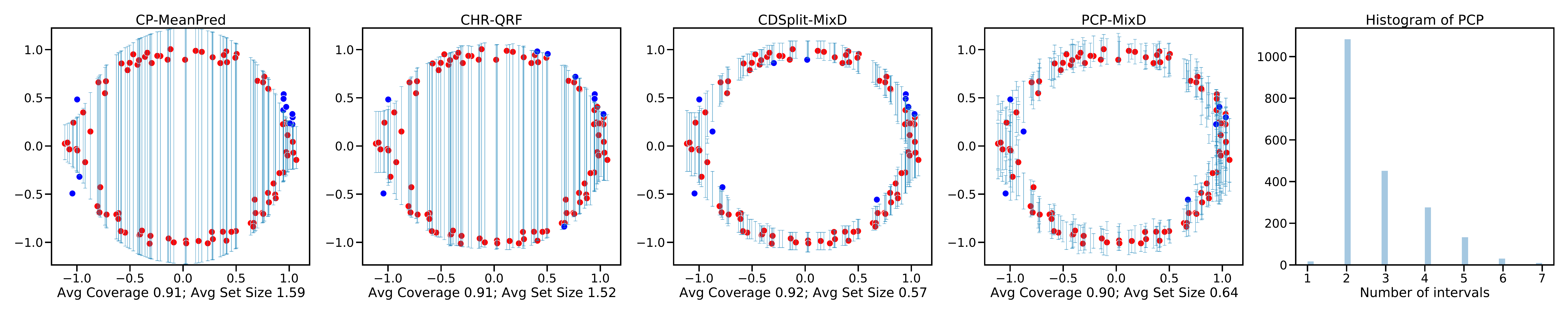} \\
    \includegraphics[width=\textwidth]{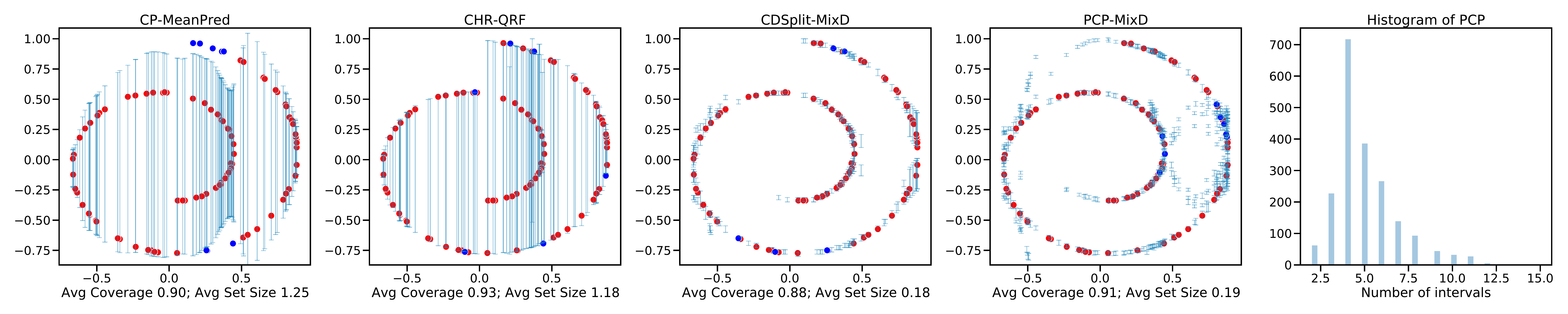} \\
    \caption{Visualization of predictive sets ($\alpha=0.1$) on 2D toy datasets: s-curve, half-moons, 25-Gaussians, 8-Gaussians, circle and swiss-roll. For ours, we show the PCP in the last two columns. We show the predictive sets on 100 test data samples, where blues lines represent the predictive sets, blue dots are test points that are not covered by the predictive sets and reds dots are the test points covered. We show the marginal coverage and the average interval length across test datapoints in the x-axis label. The fifth column shows the histogram of the number predicted intervals of PCP. }
    \label{fig:synthetic_all_fr_0.0}
\end{figure}

\begin{figure}
    \centering
    \includegraphics[width=\textwidth]{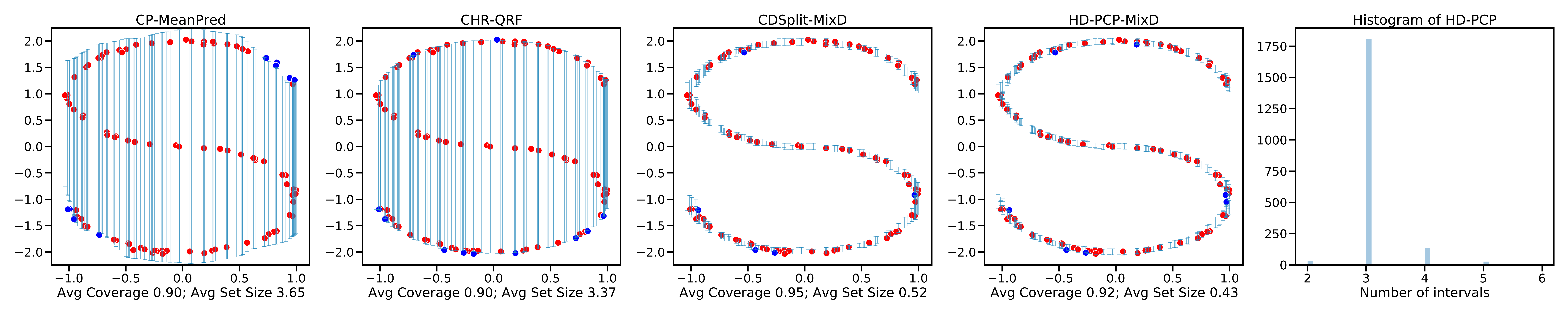} \\
    \includegraphics[width=\textwidth]{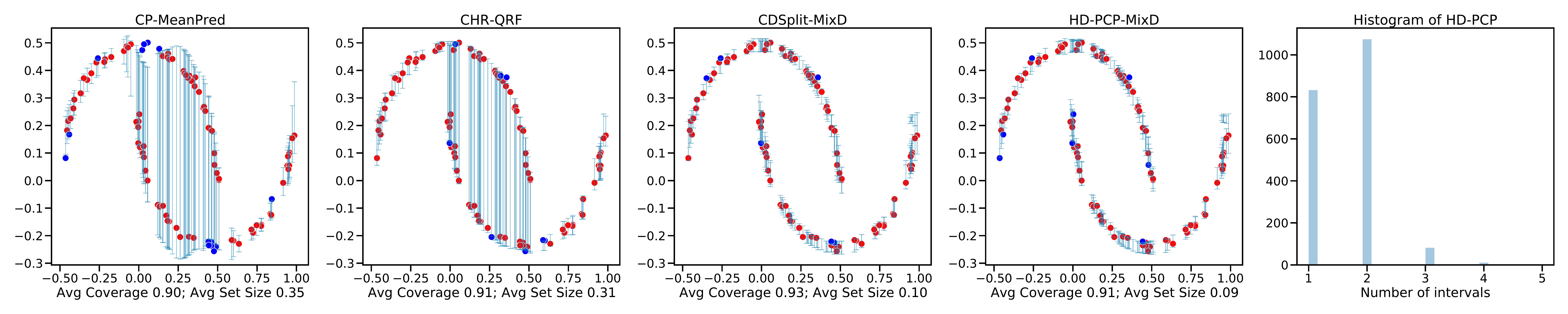} \\
    \includegraphics[width=\textwidth]{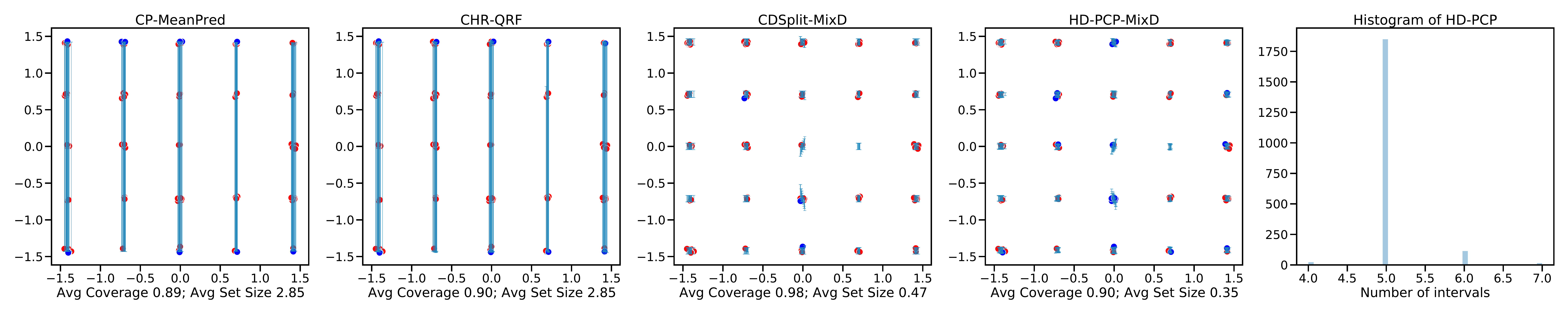} \\
    \includegraphics[width=\textwidth]{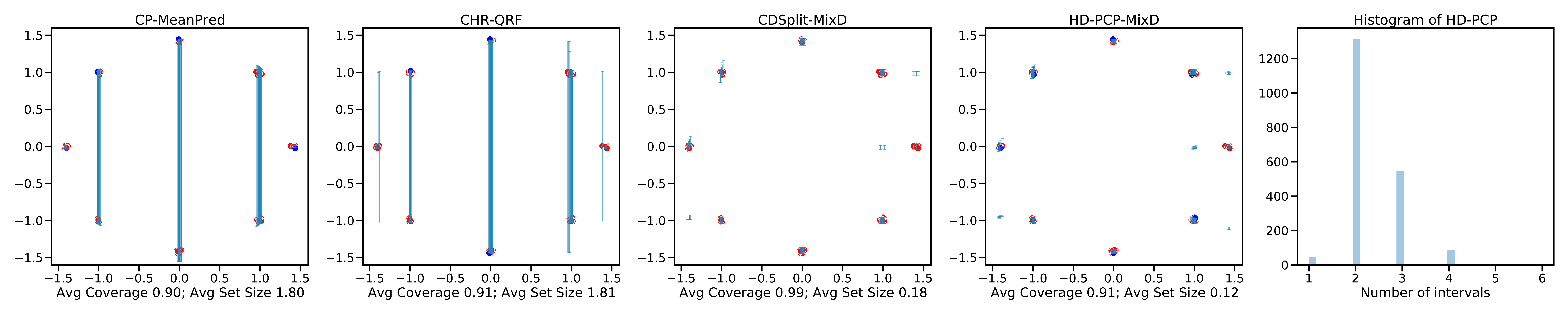} \\
    \includegraphics[width=\textwidth]{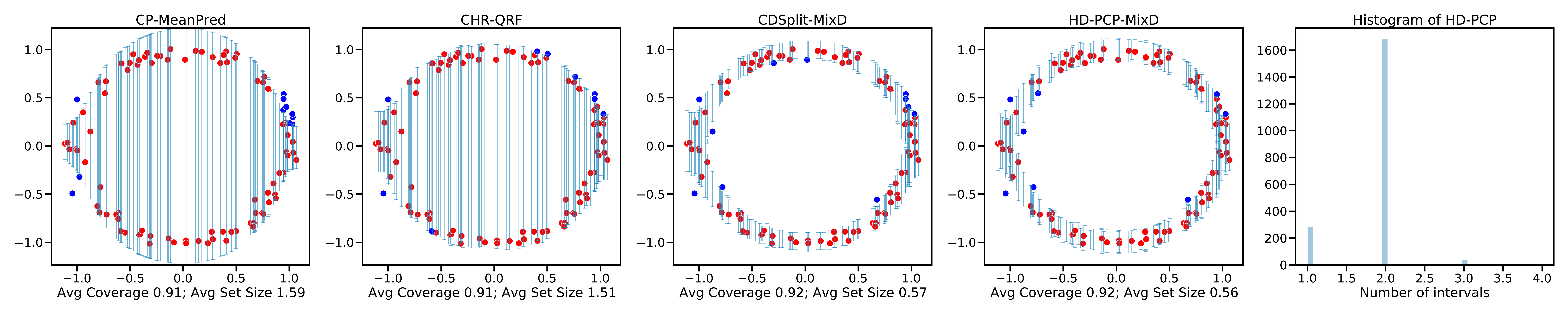} \\
    \includegraphics[width=\textwidth]{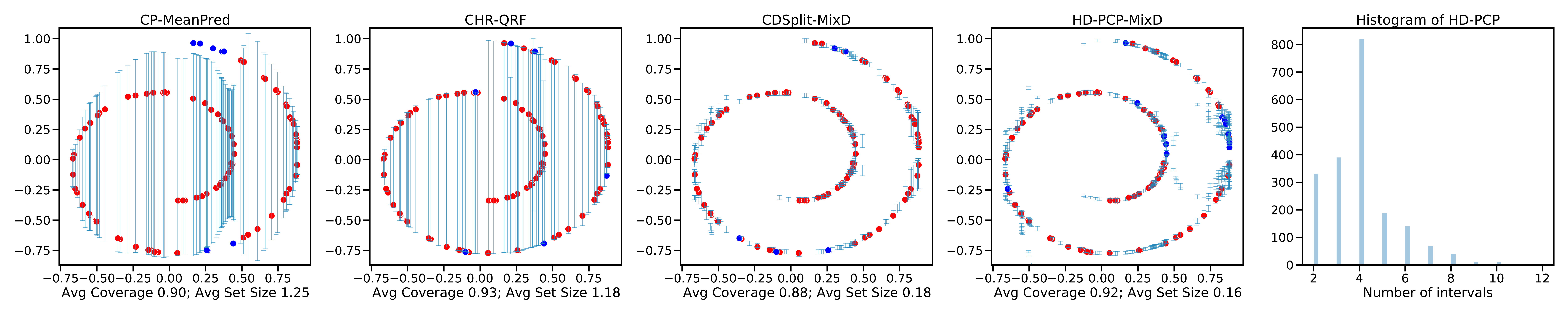} \\
    \caption{Visualization of predictive sets ($\alpha=0.1, \beta=0.2$) on 2D toy datasets: s-curve, half-moons, 25-Gaussians, 8-Gaussians, circle and swiss-roll. For ours, we show the HD-PCP in the last two columns. We show the predictive sets on 100 test data samples, where blues lines represent the predictive sets, blue dots are test points that are not covered by the predictive sets and reds dots are the test points covered. We show the marginal coverage and the average interval length across test datapoints in the x-axis label. The fifth column shows the histogram of the number predicted intervals of HD-PCP. }
    \label{fig:synthetic_all_fr_0.2}
\end{figure}

The data used to plot~\Cref{fig:mdsyndatacover} and~\Cref{fig:mdsyndataarea} is included in~\Cref{tab:md_syn_data}. When $\rho$ increases, the set size decreases for PCP and HD-PCP while keeping nearly constant for other baselines, which overlooks the joint relationship between targets. 

\section{Full real data experiment results} \label{sec:full_real_exp}

We report all experiment results for our single target real data regression tasks in \Cref{tab:real_data_1}, \Cref{tab:real_data_2} and \Cref{tab:real_data_3}. \Cref{tab:real_summary} illustrates the best performance of each method with best backbone model picked for each dataset repsectively. 

\begin{figure}
    \centering
    \includegraphics[width=\textwidth]{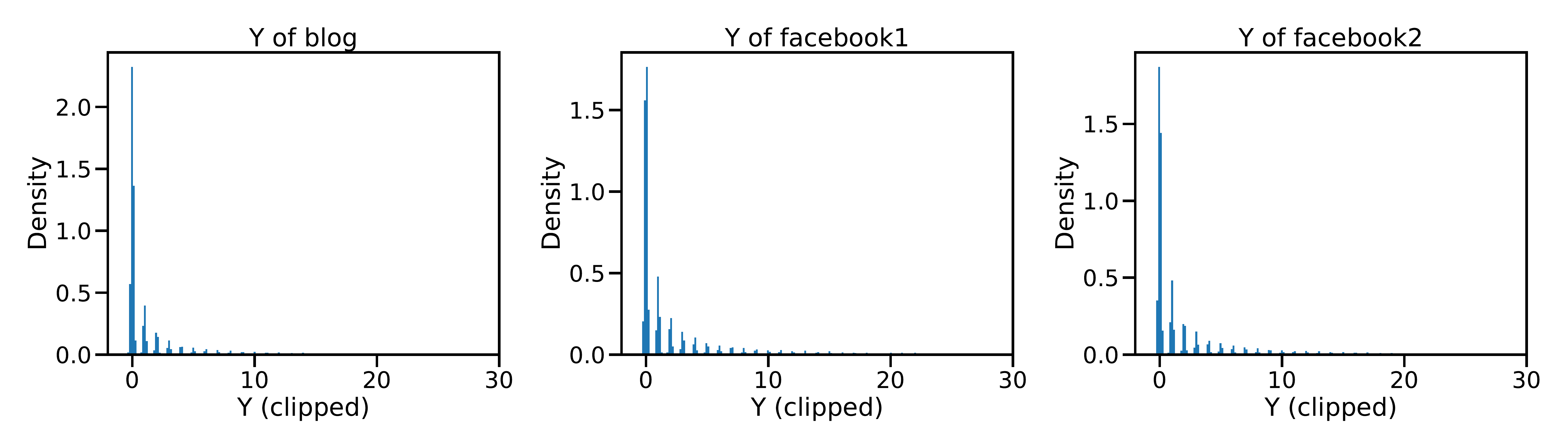} \\
    \includegraphics[width=\textwidth]{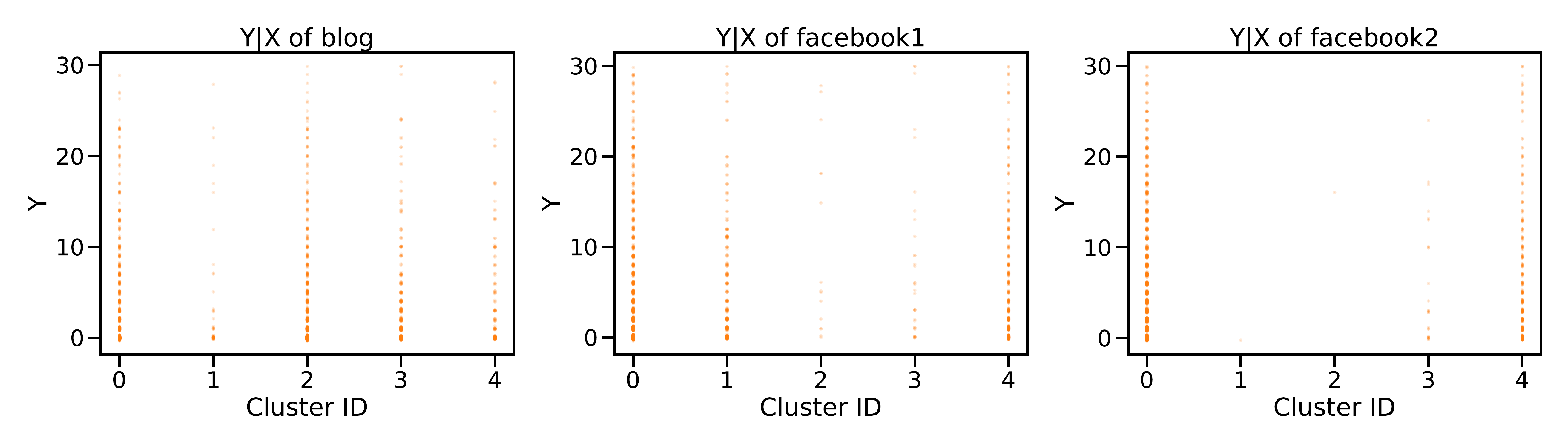}
    \caption{$Y$ response variable of blog, facebook1, and facebook2 data. The first row shows the $Y$ marginal histogram plots. The second row shows an approximation of $Y | X$ distribution, where we fit a K-Means on $X$ to generate the 5 clusters. We could find that multimodality exhibits clearly in the $Y|X$ distribution.}
    \label{fig:y_exploration}
\end{figure}

\begin{table}[]
    \centering
    \begin{tabular}{lccc}
        \toprule
        Data & n\_train & n\_calib & n\_test \\
        \midrule
        bike & 6886 & 2000 & 2000 \\
        bio & 41730 & 2000 & 2000 \\
        blog & 48397 & 2000 & 2000 \\
        facebook1 & 36948 & 2000 & 2000 \\
        facebook2 & 77311 & 2000 & 2000 \\
        meps\_19 & 11785 & 2000 & 2000 \\
        meps\_20 & 13541 & 2000 & 2000 \\
        meps\_21 & 11656 & 2000 & 2000 \\
        temperature & 5314 & 1138 & 1138 \\
        \bottomrule
    \end{tabular}
    \caption{Dataset splits for training, calibration and testing. }
    \label{tab:dataset_splits}
\end{table}

\begin{table}
    \centering
    \resizebox{0.85\textwidth}{!}{
    \begin{tabular}{llcccc}
         \toprule
         Data & Method & Marg. Coverage & Cond. Coverage &  Set Size Mean & Set Size SE \\
         \midrule
         bike & {PCP-SIVI}    & {0.90} & {0.86} & {134.55} & {1.13} \\
         \rowcolor{hl_color}
         \whitecell & \textbf{HD-PCP-MixD($\mathbf{\betav=0.2}$)} & \textbf{0.90} & \textbf{0.88} & \textbf{102.92} & \textbf{0.48} \\
         & CHR-QRF      & 0.90 & 0.88 & 204.10 & 1.03 \\
         & DistSplit   & 0.90  & 0.87 & 423.13 & 1.51 \\
         & CDSplit-MixD & 0.92 & 0.92 & 115.74 & 0.50 \\
         & DCP         & 0.90  & 0.88 & 443.76 & 1.36 \\
         & CQR         & 0.90  & 0.89 & 403.88 & 0.86 \\
         \midrule
         bio & PCP-KMN     & 0.90 & 0.89 & 10.30 & 0.05 \\
         \rowcolor{hl_color} \whitecell 
         & \textbf{HD-PCP-KMN($\mathbf{\betav=0.2}$)}  & \textbf{0.90} & \textbf{0.89} & \textbf{8.76}  & \textbf{0.05} \\
         & CHR-QRF      & 0.90 & 0.89 & 10.21 & 0.04 \\
         & DistSplit   & 0.90 & 0.89 & 13.19 & 0.04 \\
         & {CDSplit-KMN} & {0.9} & {0.9}  & {9.13}  & {0.04} \\
         & DCP         & 0.90 & 0.89 & 12.95 & 0.04 \\
         & CQR2        & 0.90 & 0.89 & 12.88 & 0.05 \\
         \midrule
         \rowcolor{hl_color} \whitecell 
         blogdata & \textbf{PCP-QRF}     & \textbf{0.89} & \textbf{0.85} & \textbf{3.68}    & \textbf{0.65} \\
         & HD-PCP-MixD($\beta=0.2$) & 0.90 & 0.87 & 9.44  & 0.19 \\
         & CHR-QRF      & 0.90  & 0.87 & 10.81   & 0.17  \\
         & DistSplit   & 0.90  & 0.87 & 16.27   & 0.23  \\
         & CDSplit-MixD & 0.96 & 0.95 & 39.00 & 0.40 \\
         & DCP         & 0.90  & 0.88 & 1422.36 & 0.03  \\
         & CQR2        & 0.90  & 0.87 & 13.91   & 0.27 \\
         \midrule
         \rowcolor{hl_color} \whitecell 
         facebook1
         & \textbf{PCP-QRF}     & \textbf{0.90}  & \textbf{0.82} & \textbf{4.52}    & \textbf{0.76} \\
         & HD-PCP-KMN($\beta=0.2$)  & 0.90 & 0.83 & 8.62 & 0.06 \\
         & CHR-NNet    & 0.90  & 0.87 & 9.96    & 0.11 \\
         & DistSplit   & 0.90  & 0.89 & 14.03   & 0.16 \\
         & CDSplit-MixD & 0.95 & 0.95 & 33.69 & 0.16 \\
         & DCP         & 0.90  & 0.89 & 1303.01 & 0.04 \\
         & CQR2        & 0.90  & 0.88 & 12.17   & 0.15 \\
        \midrule
         \rowcolor{hl_color} \whitecell 
         facebook2 
         & \textbf{PCP-QRF}     & \textbf{0.90}  & \textbf{0.82} & \textbf{3.62}    & \textbf{0.72} \\
         & HD-PCP-KMN($\beta=0.2$)  & 0.90 & 0.82 & 8.34 & 0.07 \\
         & CHR-NNet    & 0.90  & 0.87 & 10.15   & 0.13 \\
         & DistSplit   & 0.90  & 0.89 & 13.48   & 0.19 \\
         & CDSplit-KMN & 0.95 & 0.94 & 44.53   & 0.26 \\
         & DCP         & 0.90  & 0.89 & 1963.68 & 0.03 \\
         & CQR2        & 0.90  & 0.89 & 11.41   & 0.17 \\
         \midrule
         meps19 & {PCP-GAN} & {0.90}  & {0.86} & {18.41}   & {0.17} \\
         \rowcolor{hl_color} \whitecell 
         & \textbf{HD-PCP-MixD($\mathbf{\betav=0.2}$)} & \textbf{0.90} & \textbf{0.88} & \textbf{17.78} & \textbf{0.18} \\
         & {CHR-QRF}     & {0.90}  & {0.89} & {18.26}  & {0.15} \\
         & DistSplit   & 0.90  & 0.89 & 29.96  & 0.28 \\
         & CDSplit-MixD & 0.93 & 0.92 & 23.86 & 0.17 \\
         & DCP         & 0.90  & 0.88 & 559.23 & 0.01 \\
         & CQR         & 0.90  & 0.89 & 28.71  & 0.18 \\
        \midrule
        meps20 & PCP-MixD    & 0.90 & 0.87 & 19.52 & 0.16 \\
        \rowcolor{hl_color} \whitecell 
         & \textbf{HD-PCP-MixD($\mathbf{\betav=0.2}$)} & \textbf{0.90} & \textbf{0.88} & \textbf{18.19} & \textbf{0.17} \\
         \rowcolor{hl_color} \whitecell
         & \textbf{CHR-QRF}      & \textbf{0.90}  & \textbf{0.90}  & \textbf{17.94}  & \textbf{0.18} \\
         & DistSplit   & 0.90  & 0.90  & 29.35  & 0.23 \\
         & CDSplit-MixD & 0.92 & 0.92 & 22.93 & 0.16 \\
         & DCP         & 0.90  & 0.88 & 520.25 & 0.01 \\
         & CQR         & 0.90  & 0.90  & 27.57  & 0.15 \\
         \midrule
         meps21 & PCP-MixD  & 0.90 & 0.87 & 19.18 & 0.12 \\
         \rowcolor{hl_color} \whitecell 
         & \textbf{HD-PCP-MixD($\mathbf{\betav=0.2}$)} & \textbf{0.90} & \textbf{0.88} & \textbf{17.91} & \textbf{0.15} \\
         & {CHR-QRF}      & {0.90}  & {0.90}  & {18.65}  & {0.16} \\
         & DistSplit   & 0.90  & 0.89 & 30.32  & 0.31 \\
         & CDSplit-MixD & 0.93 & 0.92 & 23.63 & 0.17 \\
         & DCP         & 0.90  & 0.88 & 531.25 & 0.01 \\
         & CQR         & 0.90  & 0.89 & 29.89  & 0.2  \\
         \midrule
         temperature & PCP-MixD    & 0.90 & 0.90 & 2.10 & 0.01 \\
         \rowcolor{hl_color} \whitecell 
         & \textbf{HD-PCP-MixD($\mathbf{\betav=0.2}$)} & \textbf{0.90} & \textbf{0.89} & \textbf{1.85} & \textbf{0.01} \\
         & CHR-NNet    & 0.90  & 0.89 & 3.17  & 0.01 \\
         & DistSplit   & 0.90  & 0.89 & 3.07  & 0.01 \\
         & CDSplit-MixD & 0.92 & 0.91 & 2.23 & 0.01 \\
         & DCP         & 0.90  & 0.88 & 3.1   & 0.02 \\
         & CQR2        & 0.90  & 0.88 & 3.14  & 0.02 \\
         \bottomrule
    \end{tabular}}
    \caption{Best results of real data experiments (the best variant of each method for each dataset is selected). }
    \label{tab:real_summary}
\end{table}

\begin{table}
    \centering
    \resizebox{0.85\textwidth}{!}{
    \begin{tabular}{llcccc}
         \toprule
         Data & Method & Marg. Coverage & Cond. Coverage &  Set Size Mean & Set Size SE \\
         \midrule
         bike & {PCP-SIVI}    & {0.90} & {0.86} & {134.55} & {1.13} \\
         & PCP-GAN     & 0.90 & 0.88 & 399.32 & 2.48 \\
         & PCP-QRF     & 0.90 & 0.87 & 241.11 & 3.35 \\
         & PCP-MixD    & 0.90 & 0.87 & 128.13 & 0.53 \\
         \rowcolor{hl_color} \whitecell
         & \textbf{HD-PCP-MixD($\mathbf{\betav=0.2}$)} & \textbf{0.90} & \textbf{0.88} & \textbf{102.92} & \textbf{0.48} \\
         & PCP-KMN     & 0.90 & 0.88 & 172.92 & 1.04 \\
         & HD-PCP-KMN($\beta=0.2$)  & 0.90 & 0.88 & 146.24 & 1.06 \\
         & CHR-NNet    & 0.90 & 0.89 & 353.51 & 1.59 \\
         & CHR-QRF      & 0.90 & 0.88 & 204.10 & 1.03 \\
         & DistSplit   & 0.90 & 0.87 & 423.13 & 1.51 \\
         & CDSplit-KMN & 0.92 & 0.91 & 161.16 & 0.72 \\
         & CDSplit-MixD & 0.92 & 0.92 & 115.74 & 0.50 \\
         & DCP         & 0.90 & 0.88 & 443.76 & 1.36 \\
         & CQR         & 0.90 & 0.89 & 403.88 & 0.86 \\
         & CQR2        & 0.90 & 0.88 & 416.75 & 1.57 \\
         \midrule
         bio & PCP-SIVI    & 0.90 & 0.89 & 14.08 & 0.06 \\
         & PCP-GAN     & 0.90 & 0.89 & 13.11 & 0.05 \\
         & PCP-QRF     & 0.90 & 0.89 & 11.08 & 0.16 \\
         & PCP-MixD    & 0.90 & 0.89 & 11.47 & 0.04 \\
         & HD-PCP-MixD($\beta=0.2$) & 0.90 & 0.90 & 10.06 & 0.05 \\
         & PCP-KMN     & 0.90 & 0.89 & 10.30 & 0.05 \\
         \rowcolor{hl_color} \whitecell
         & \textbf{HD-PCP-KMN($\mathbf{\betav=0.2}$)}  & \textbf{0.90} & \textbf{0.89} & \textbf{8.76}  & \textbf{0.05} \\
         & CHR-NNet    & 0.90 & 0.89 & 11.74 & 0.04 \\
         & CHR-QRF      & 0.90 & 0.89 & 10.21 & 0.04 \\
         & DistSplit   & 0.90 & 0.89 & 13.19 & 0.04 \\
         & {CDSplit-KMN} & {0.90} & {0.90} & {9.13}  & {0.04} \\
         & CDSplit-MixD & 0.90 & 0.90 & 9.58 & 0.04 \\
         & DCP         & 0.90 & 0.89 & 12.95 & 0.04 \\
         & CQR         & 0.90 & 0.89 & 13.00 & 0.02 \\
         & CQR2        & 0.90 & 0.89 & 12.88 & 0.05 \\
         \midrule
         blog & PCP-SIVI    & 0.90  & 0.85 & 11.21   & 0.32  \\
         & PCP-GAN  & 0.90  & 0.86 & 11.67   & 0.16  \\
         \rowcolor{hl_color}
         \whitecell & \textbf{PCP-QRF}     & \textbf{0.89} & \textbf{0.85} & \textbf{3.68}    & \textbf{0.65} \\
         & PCP-MixD    & 0.90 & 0.85 & 10.78 & 0.17 \\
         
         & HD-PCP-MixD($\beta=0.2$) & 0.90 & 0.87 & 9.44  & 0.19 \\
         & PCP-KMN     & 0.90 & 0.85 & 10.67 & 0.13 \\
         & HD-PCP-KMN($\beta=0.2$)  & 0.90 & 0.86 & 10.51 & 0.14 \\
         & CHR-NNet    & 0.90  & 0.88 & 11.1    & 0.19  \\
         & CHR-QRF      & 0.90  & 0.87 & 10.81   & 0.17  \\
         & DistSplit   & 0.90  & 0.87 & 16.27   & 0.23  \\
         & CDSplit-KMN & 0.96 & 0.95 & 45.90    & 0.62  \\
         & CDSplit-MixD & 0.96 & 0.95 & 39.00 & 0.40 \\
         & DCP         & 0.90  & 0.88 & 1422.36 & 0.03  \\
         & CQR         & 0.90  & 0.87 & 15.15   & 0.26  \\
         & CQR2        & 0.90  & 0.87 & 13.91   & 0.27 \\
         \midrule
         facebook1 & PCP-SIVI    & 0.90  & 0.83 & 8.8     & 0.06 \\
         & PCP-GAN     & 0.90  & 0.85 & 9.22    & 0.05 \\
         \rowcolor{hl_color} \whitecell
         & \textbf{PCP-QRF}     & \textbf{0.90}  & \textbf{0.82} & \textbf{4.52}    & \textbf{0.76} \\
         & PCP-MixD    & 0.90 & 0.82 & 9.99 & 0.14 \\
         & HD-PCP-MixD($\beta=0.2$) & 0.90 & 0.85 & 8.93 & 0.12 \\
         & PCP-KMN     & 0.90 & 0.82 & 10.60 & 0.06 \\
         & HD-PCP-KMN($\beta=0.2$)  & 0.90 & 0.83 & 8.62 & 0.06 \\
         & CHR-NNet    & 0.90  & 0.87 & 9.96    & 0.11 \\
         & CHR-QRF      & 0.90  & 0.86 & 11.21   & 0.12 \\
         & DistSplit   & 0.90  & 0.89 & 14.03   & 0.16 \\
         & CDSplit-KMN & 0.95 & 0.94 & 33.88   & 0.19 \\
         & CDSplit-MixD & 0.95 & 0.95 & 33.69 & 0.16 \\
         & DCP         & 0.90  & 0.89 & 1303.01 & 0.04 \\
         & CQR         & 0.90  & 0.89 & 13.79   & 0.15 \\
         & CQR2        & 0.90  & 0.88 & 12.17   & 0.15 \\
         \bottomrule
    \end{tabular}}
    \caption{Detailed results of experiments on data: bike, bio, blog and facebook1. }
    \label{tab:real_data_1}
\end{table}

\begin{table}
    \centering
    \resizebox{0.85\textwidth}{!}{
    \begin{tabular}{llcccc}
         \toprule
         Data & Method & Marg. Coverage & Cond. Coverage &  Set Size Mean & Set Size SE\\
         \midrule
         facebook2 & PCP-SIVI    & 0.90  & 0.83 & 8.69    & 0.17 \\
         & PCP-GAN     & 0.90  & 0.84 & 9.47    & 0.1  \\
         \rowcolor{hl_color} \whitecell
         & \textbf{PCP-QRF}     & \textbf{0.90}  & \textbf{0.82} & \textbf{3.62}    & \textbf{0.72} \\
         & PCP-MixD    & 0.90 & 0.82 & 9.93 & 0.11 \\
         & HD-PCP-MixD($\beta=0.2$) & 0.90 & 0.84 & 8.84 & 0.10 \\
         & PCP-KMN     & 0.90 & 0.81 & 10.42 & 0.07 \\
         & HD-PCP-KMN($\beta=0.2$)  & 0.90 & 0.82 & 8.34 & 0.07 \\
         & CHR-NNet    & 0.90  & 0.87 & 10.15   & 0.13 \\
         & CHR-QRF      & 0.90  & 0.87 & 10.81   & 0.14 \\
         & DistSplit   & 0.90  & 0.89 & 13.48   & 0.19 \\
         & CDSplit-KMN & 0.95 & 0.94 & 44.53   & 0.26 \\
         & CDSplit-MixD & 0.97 & 0.96 & 45.75 & 0.16 \\
         & DCP         & 0.90  & 0.89 & 1963.68 & 0.03 \\
         & CQR         & 0.90  & 0.89 & 13      & 0.17 \\
         & CQR2        & 0.90  & 0.89 & 11.41   & 0.17 \\
         \midrule
         meps19 & PCP-SIVI    & 0.90  & 0.85 & 26.93  & 0.3  \\
         & {PCP-GAN} & {0.90}  & {0.86} & {18.41}   & {0.17} \\
         & PCP-QRF  & 0.90  & 0.86 & 20.16  & 0.48 \\
         & PCP-MixD    & 0.90 & 0.87 & 19.28 & 0.16 \\
         \rowcolor{hl_color} \whitecell
         & \textbf{HD-PCP-MixD($\mathbf{\betav=0.2}$)} & \textbf{0.90} & \textbf{0.88} & \textbf{17.78} & \textbf{0.18} \\
         & PCP-KMN     & 0.90 & 0.85 & 23.24 & 0.21 \\
         & HD-PCP-KMN($\beta=0.2$)  & 0.90 & 0.84 & 23.48 & 0.20 \\
         & CHR-NNet & 0.90  & 0.90  & 20.17  & 0.2  \\
         & {CHR-QRF} & {0.90}  & {0.89} & {18.26}  & {0.15} \\
         & DistSplit   & 0.90  & 0.89 & 29.96  & 0.28 \\
         & CDSplit-KMN & 0.93 & 0.91 & 31.10  & 0.33 \\
         & CDSplit-MixD & 0.93 & 0.92 & 23.86 & 0.17 \\
         & DCP         & 0.90  & 0.88 & 559.23 & 0.01 \\
         & CQR         & 0.90  & 0.89 & 28.71  & 0.18 \\
         & CQR2        & 0.90  & 0.89 & 30.78  & 0.36 \\
        \midrule
        meps20 & PCP-SIVI    & 0.90  & 0.86 & 23.87  & 0.16 \\
         & PCP-GAN     & 0.90  & 0.86 & 19.92  & 0.18 \\
         & PCP-QRF     & 0.90  & 0.86 & 20.47  & 0.52 \\
         & PCP-MixD    & 0.90 & 0.87 & 19.52 & 0.16 \\
         \rowcolor{hl_color} \whitecell
         & \textbf{HD-PCP-MixD($\mathbf{\betav=0.2}$)} & \textbf{0.90} & \textbf{0.88} & \textbf{18.19} & \textbf{0.17} \\
         & PCP-KMN     & 0.90 & 0.85 & 22.96 & 0.17 \\
         & HD-PCP-KMN($\beta=0.2$)  & 0.90 & 0.85 & 23.35 & 0.18 \\
         & CHR-NNet    & 0.90  & 0.9  & 19.43  & 0.18 \\
         \rowcolor{hl_color} \whitecell
         & \textbf{CHR-QRF}      & \textbf{0.90}  & \textbf{0.90}  & \textbf{17.94}  & \textbf{0.18} \\
         & DistSplit   & 0.90  & 0.90  & 29.35  & 0.23 \\
         & CDSplit-KMN & 0.93 & 0.90  & 29.05  & 0.30  \\
         & CDSplit-MixD & 0.92 & 0.92 & 22.93 & 0.16 \\
         & DCP         & 0.90  & 0.88 & 520.25 & 0.01 \\
         & CQR         & 0.90  & 0.90  & 27.57  & 0.15 \\
         & CQR2        & 0.90  & 0.90  & 29.94  & 0.31 \\
         \midrule
         meps21 & PCP-SIVI    & 0.90  & 0.85 & 23.74  & 0.27 \\
         & PCP-GAN     & 0.90  & 0.86 & 19.73  & 0.16 \\
         & {PCP-QRF}     & {0.89} & {0.86} & {18.52} & {0.45} \\
         & PCP-MixD    & 0.90 & 0.87 & 19.18 & 0.12 \\
         \rowcolor{hl_color} \whitecell
         & \textbf{HD-PCP-MixD($\mathbf{\betav=0.2}$)} & \textbf{0.90} & \textbf{0.88} & \textbf{17.91} & \textbf{0.15} \\
         & PCP-KMN     & 0.90 & 0.85 & 23.13 & 0.17 \\
         & HD-PCP-KMN($\beta=0.2$)  & 0.90 & 0.85 & 23.70 & 0.19 \\
         & CHR-NNet    & 0.90  & 0.90  & 20.07  & 0.22 \\
         & {CHR-QRF}      & {0.90}  & {0.90}  & {18.65} & {0.16} \\
         & DistSplit   & 0.90  & 0.89 & 30.32  & 0.31 \\
         & CDSplit-KMN & 0.92 & 0.91 & 30.42  & 0.41 \\
         & CDSplit-MixD & 0.93 & 0.92 & 23.63 & 0.17 \\
         & DCP         & 0.90  & 0.88 & 531.25 & 0.01 \\
         & CQR         & 0.90  & 0.89 & 29.89  & 0.2  \\
         & CQR2        & 0.90  & 0.89 & 31.78  & 0.36 \\
         \bottomrule
    \end{tabular}}
    \caption{Detailed results of experiments on data: facebook2, meps19, meps20 and meps21.}
    \label{tab:real_data_2}
\end{table}

\begin{table}
    \centering
    \resizebox{0.9\textwidth}{!}{
    \begin{tabular}{llcccc}
         \toprule
         Data & Method & Marg. Coverage & Cond. Coverage &  Set Size Mean & Set Size SE\\
         \midrule
        temperature & PCP-SIVI    & 0.90  & 0.90  & 3.27 & 0.06 \\
         & PCP-GAN     & 0.90  & 0.89 & 3.51  & 0.04 \\
         & PCP-QRF     & 0.88 & 0.86 & 3.78  & 0.09 \\
         & PCP-MixD    & 0.90 & 0.90 & 2.10 & 0.01 \\
         \rowcolor{hl_color} \whitecell
         & \textbf{HD-PCP-MixD($\mathbf{\betav=0.2}$)} & \textbf{0.90} & \textbf{0.89} & \textbf{1.85} & \textbf{0.01} \\
         & PCP-KMN     & 0.90 & 0.89 & 2.68 & 0.01 \\
         & HD-PCP-KMN($\beta=0.2$) & 0.90 & 0.89 & 2.43 & 0.01 \\
         & CHR-NNet    & 0.90  & 0.89 & 3.17  & 0.01 \\
         & CHR-QRF      & 0.90 & 0.89 & 3.24  & 0.01 \\
         & DistSplit   & 0.90  & 0.89 & 3.07  & 0.01 \\
         & CDSplit-KMN & 0.91 & 0.90  & 2.84  & 0.02 \\
         & CDSplit-MixD & 0.92 & 0.91 & 2.23 & 0.01 \\
         & DCP         & 0.90  & 0.88 & 3.1   & 0.02 \\
         & CQR         & 0.90  & 0.87 & 3.55  & 0.03 \\
         & CQR2        & 0.90  & 0.88 & 3.14  & 0.02 \\
         \bottomrule
    \end{tabular}}
    \caption{Detailed results of experiments on data: temperature.}
    \label{tab:real_data_3}
\end{table}

\begin{table}
    \centering
    \resizebox{0.9\textwidth}{!}{
    \begin{tabular}{llcccc}
         \toprule
         Synthetic Data & Method & Cond. Coverage & Marg. Coverage &  Set Size Mean & Set Size SE \\
         \midrule
         \rowcolor{hl_color} \whitecell $\rho=0$
         & HD-PCP-MixD  & 0.92 (0.03)  & 0.90 (0.01) & 356.82 & 10.78 \\    
         & PCP-MixD   & 0.94 (0.02)  & 0.90 (0.01) & 412.40 & 12.70 \\
         & CHR-NNet    & 0.88 (0.05)  & 0.92 (0.01) & 690.49 & 51.55 \\
         & DistSplit   & 0.93 (0.01)  & 0.92 (0.01) & 714.92 & 54.81 \\
         & CDSplit-MixD & 0.87 (0.03)  & 0.90 (0.01) & 437.98 & 20.20 \\
         & DCP         & 0.90 (0.03)  & 0.92 (0.01) & 710.64 & 53.99 \\
         & CQR         & 0.95 (0.02)  & 0.92 (0.01) & 667.01 & 51.47 \\
         & CQR2         & 0.89 (0.03)  & 0.92 (0.01) & 694.64 & 52.40 \\
         \midrule
         \rowcolor{hl_color} \whitecell $\rho=5$
         & HD-PCP-MixD & 0.88 (0.03)  & 0.89 (0.01) & 302.64 & 9.88 \\    
         & PCP-MixD   & 0.87 (0.03)  & 0.88 (0.01) & 348.93 & 10.70 \\ 
         & CHR-NNet    & 0.92 (0.03)  & 0.91 (0.01) & 689.22 & 57.47 \\
         & DistSplit   & 0.93 (0.02)  & 0.91 (0.01) & 702.56 & 54.02 \\
         & CDSplit-MixD & 0.87 (0.03)  & 0.90 (0.01) & 415.86 & 17.49 \\
         & DCP         & 0.89 (0.05)  & 0.91 (0.02) & 697.18 & 54.01 \\
         & CQR         & 0.87 (0.03)  & 0.91 (0.02) & 653.64 & 53.59 \\
         & CQR2         & 0.88 (0.05)  & 0.91 (0.02) & 693.10 & 56.18 \\         
         \midrule
         \rowcolor{hl_color} \whitecell $\rho=9$
         & HD-PCP-MixD & 0.83 (0.07)  & 0.89 (0.02) & 171.61 & 8.36 \\    
         & PCP-MixD   & 0.89 (0.03)  & 090 (0.01) & 205.63 & 5.76 \\          
         & CHR-NNet    & 0.92 (0.02)  & 0.92 (0.02) & 664.31 & 51.59 \\
         & DistSplit   & 0.91 (0.05)  & 0.91 (0.01) & 689.72 & 59.96 \\
         & CDSplit-MixD & 0.92 (0.02)  & 0.91 (0.01) & 416.86 & 16.96 \\
         & DCP         & 0.90 (0.03)  & 0.91 (0.02) & 666.13 & 50.61 \\
         & CQR         & 0.93 (0.04)  & 0.92 (0.01) & 639.00 & 50.19 \\
         & CQR2         & 0.92 (0.02)  & 0.91 (0.01) & 663.54 & 54.59 \\         
         \bottomrule
    \end{tabular}}
    \caption{Detailed results for multidimensional target synthetic dataset. The set size for PCP and HD-PCP decreases when $\rho$ increases while the set sizes for other baselines are similar for different $\rho$ since the marginal distribution remains the same.}
    \label{tab:md_syn_data}
\end{table}

\section{Additional Plots for NYC Taxi Data} \label{sec:taxiapp}
\Cref{fig:taxiaddtion_app} shows the additonal plots for NYC Taxi data for PCP, HD-PCP, CHR and CDSplit. Each row representsone individual record in the test set and we show the predictive set generated by each algorithm. Clearly, PCP and HD-PCP generate the most informative predictive set where popular neighborhoods and airports are tagged, while HD-PCP offers a more sparse and clean set. As expected, CHR offers a wide predictive set. Both CDSplit and CHR fail to provide predictive set with clear interpretation. \Cref{fig:taxi_app} shows the predictive set and heatmap for pickup from SOHO and Chinatown. Most popular spots in NYC have higher density, which means passengers are more likely to be dropped off there. 

\begin{figure}[ht]
    \centering
     \begin{subfigure}[b]{0.24\textwidth}
         \includegraphics[width=\textwidth]{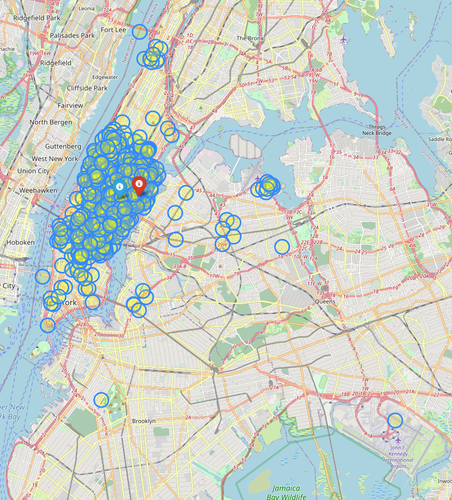}
     \caption{P1: PCP}
     \end{subfigure} 
     \begin{subfigure}[b]{0.24\textwidth}
         \includegraphics[width=\textwidth]{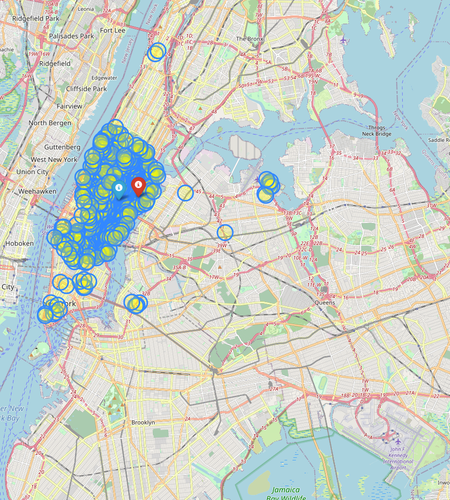}
     \caption{P1: HD-PCP}
     \end{subfigure}
     \begin{subfigure}[b]{0.24\textwidth}
         \includegraphics[width=\textwidth]{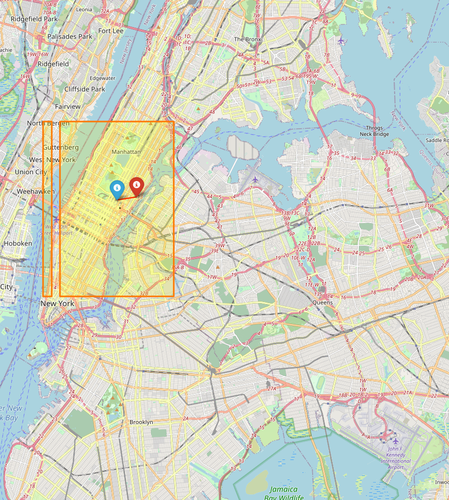}
     \caption{P1: CDSplit}
     \end{subfigure}
     \begin{subfigure}[b]{0.24\textwidth}
         \includegraphics[width=\textwidth]{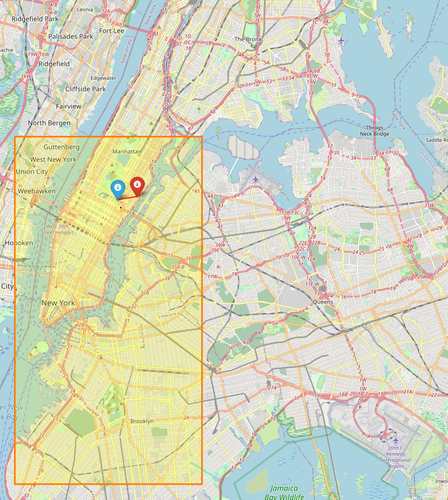}
     \caption{P1: CHR}
     \end{subfigure} \\
    \begin{subfigure}[b]{0.24\textwidth}
        \includegraphics[width=\textwidth]{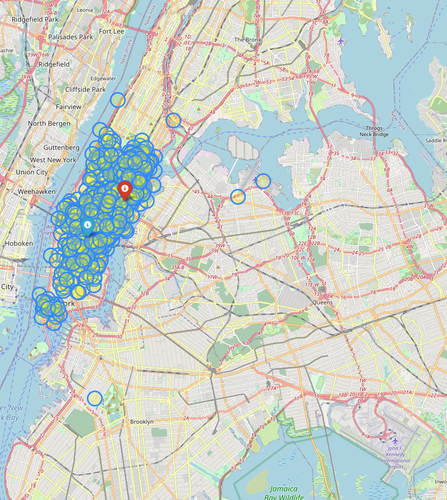}
     \caption{P2: PCP}
     \end{subfigure} 
     \begin{subfigure}[b]{0.24\textwidth}
         \includegraphics[width=\textwidth]{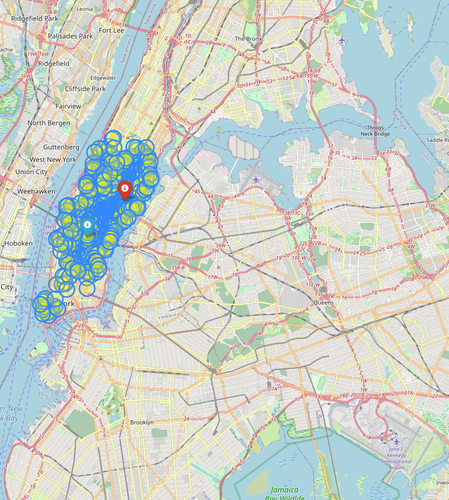}
     \caption{P2: HD-PCP}
     \end{subfigure}
     \begin{subfigure}[b]{0.24\textwidth}
         \includegraphics[width=\textwidth]{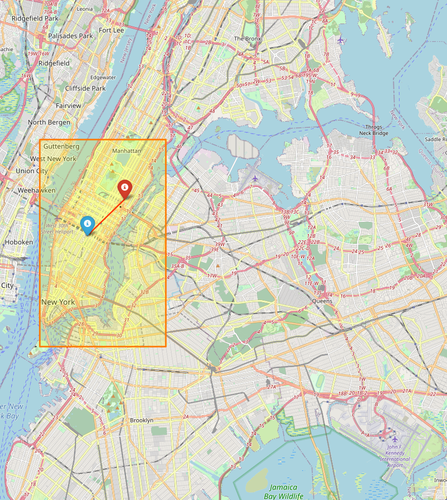}
     \caption{P2: CDSplit}
     \end{subfigure}
     \begin{subfigure}[b]{0.24\textwidth}
         \includegraphics[width=\textwidth]{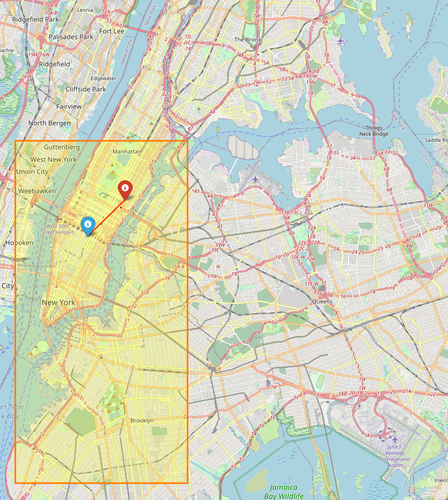}
     \caption{P2: CHR}
     \end{subfigure} \\
    \begin{subfigure}[b]{0.24\textwidth}
        \includegraphics[width=\textwidth]{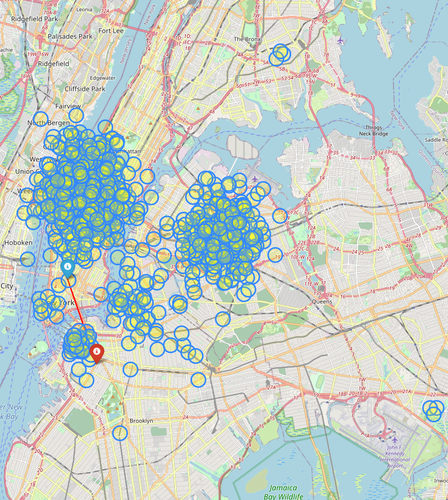}
     \caption{P3: PCP}
     \end{subfigure} 
     \begin{subfigure}[b]{0.24\textwidth}
         \includegraphics[width=\textwidth]{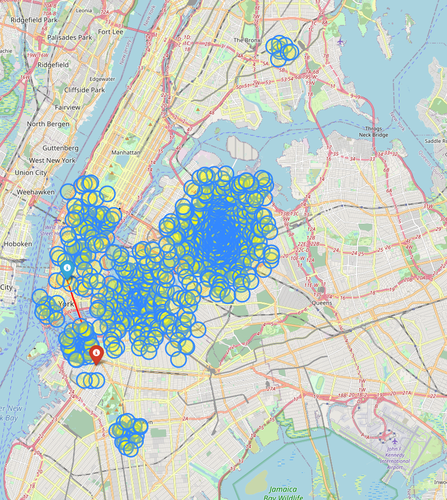}
     \caption{P3: HD-PCP}
     \end{subfigure}
     \begin{subfigure}[b]{0.24\textwidth}
         \includegraphics[width=\textwidth]{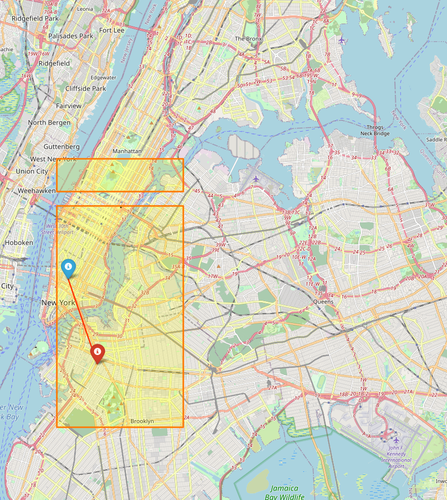}
     \caption{P3: CDSplit}
     \end{subfigure}
     \begin{subfigure}[b]{0.24\textwidth}
         \includegraphics[width=\textwidth]{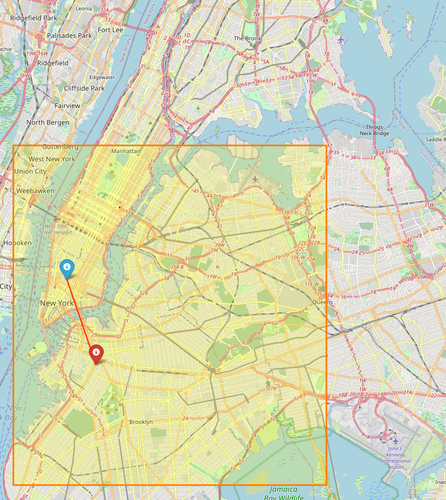}
     \caption{P3: CHR}
     \end{subfigure}
    \caption{NYC Taxi data. Red Pin: pickup location; Blue Pin: dropoff location. Left to right: Predictive set output by PCP, HD-PCP, CDSplit and CHR. Each row represents a random selected individual record sampled from the dataset.}
    \label{fig:taxiaddtion_app}
\end{figure}

\begin{figure}[ht]
    \centering
     \begin{subfigure}[b]{0.24\textwidth}
         \includegraphics[width=\textwidth]{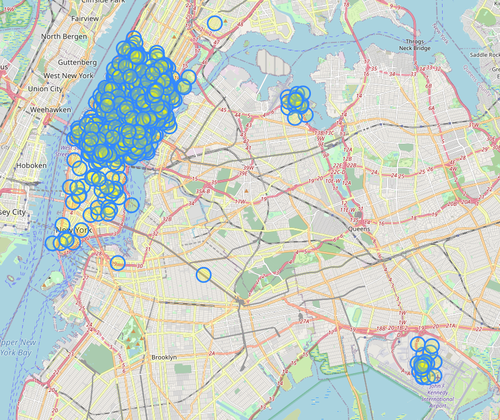}
     \caption{PCP}
     \label{fig:taxipcp}
     \end{subfigure} 
     \begin{subfigure}[b]{0.24\textwidth}
         \includegraphics[width=\textwidth]{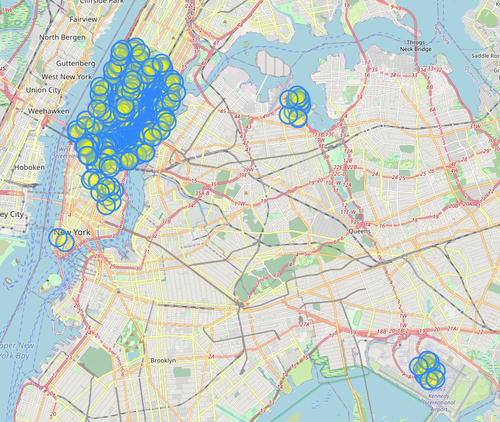}
     \caption{HD-PCP}
     \label{fig:taxihdpcp}
     \end{subfigure}
     \begin{subfigure}[b]{0.24\textwidth}
         \includegraphics[width=\textwidth]{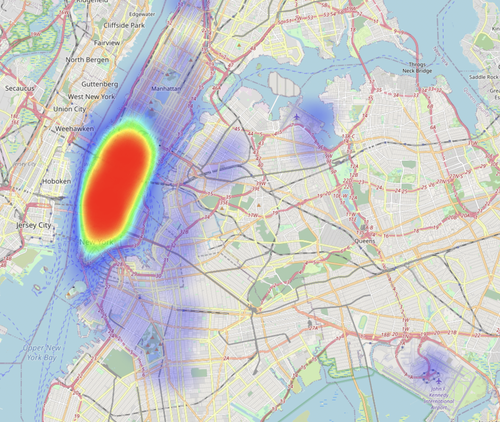}
     \caption{SOHO}
     \label{fig:taxisoho}
     \end{subfigure}
     \begin{subfigure}[b]{0.24\textwidth}
         \includegraphics[width=\textwidth]{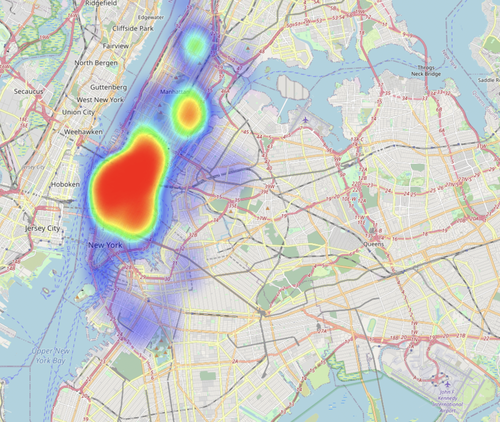}
     \caption{Chinatown}
     \label{fig:taxichinatown}
     \end{subfigure}
    \caption{NYC Taxi data. (a), (b): predictive set for an individual using PCP and HD-PCP; (c): predictive set for riders from SoHo; (d): predictive set for riders from Chinatown.}
    \label{fig:taxi_app}
\end{figure}

\section{Additional Results for Multi-Target Regression Task}\label{sec:addmultiapp}

We include more experiment results for multi-target regression task in this section. We use two datasets for river flow prediction \citep{spyromitros2016multi} and stock prediction from StatLib repository. River flow dataset predicts the rivernetwork flows for future 48 hours for 8 sites (8 targets) and the stock dataset has stock price for 10 aerospace companies and we try to predict 3 companies' price using remaining companies'. There are 64 features including past river flow information for river flow predictions. For train, calibration and test size, we use 6925, 2000 and 200 for river flow prediction and 750, 100 and 100 for stock prediction respectively.

Since the Monte-Carlo estimation of overlapping hypersphere suffers from curse of dimensionality. We convert each dataset into two-dimensional pairwise comparisions to evaluate the robustness of each method (8 targets result in 28 pairs). We plot the pairwise comparison of PCP and HD-PCP against CHR and CDSplit, the two baselines that performs generally the best among other datasets. We use Mixture density Network for CDSplit, PCP and HD-PCP and Neural Netork based CHR, the results are averaged over 5 runs. 

For X-axis, we plot the set size of PCP / HD-PCP and for Y-axis, we plot the set size for CDSplit and CHR, and we also show the $Y=X$ line. If all points fall into the left region, it means PCP / HD-PCP outputs a sharper predictive set. For PCP, almost all pointsfall into the left region, which indicates PCP has a better or comparable performance with CDSplit and CHR in all pairwise comparisions. HD-PCP has a much better performance and the points all fall into the far left part in the figure, which shows HD-PCP offers a much sharper predictive set. 

\begin{figure}[ht]
    \centering
     \begin{subfigure}[b]{0.24\textwidth}
      \makebox[0pt][r]{\makebox[20pt]{\raisebox{30pt}{\rotatebox[origin=c]{90}{River Flow}}}}%
         \includegraphics[width=\textwidth]{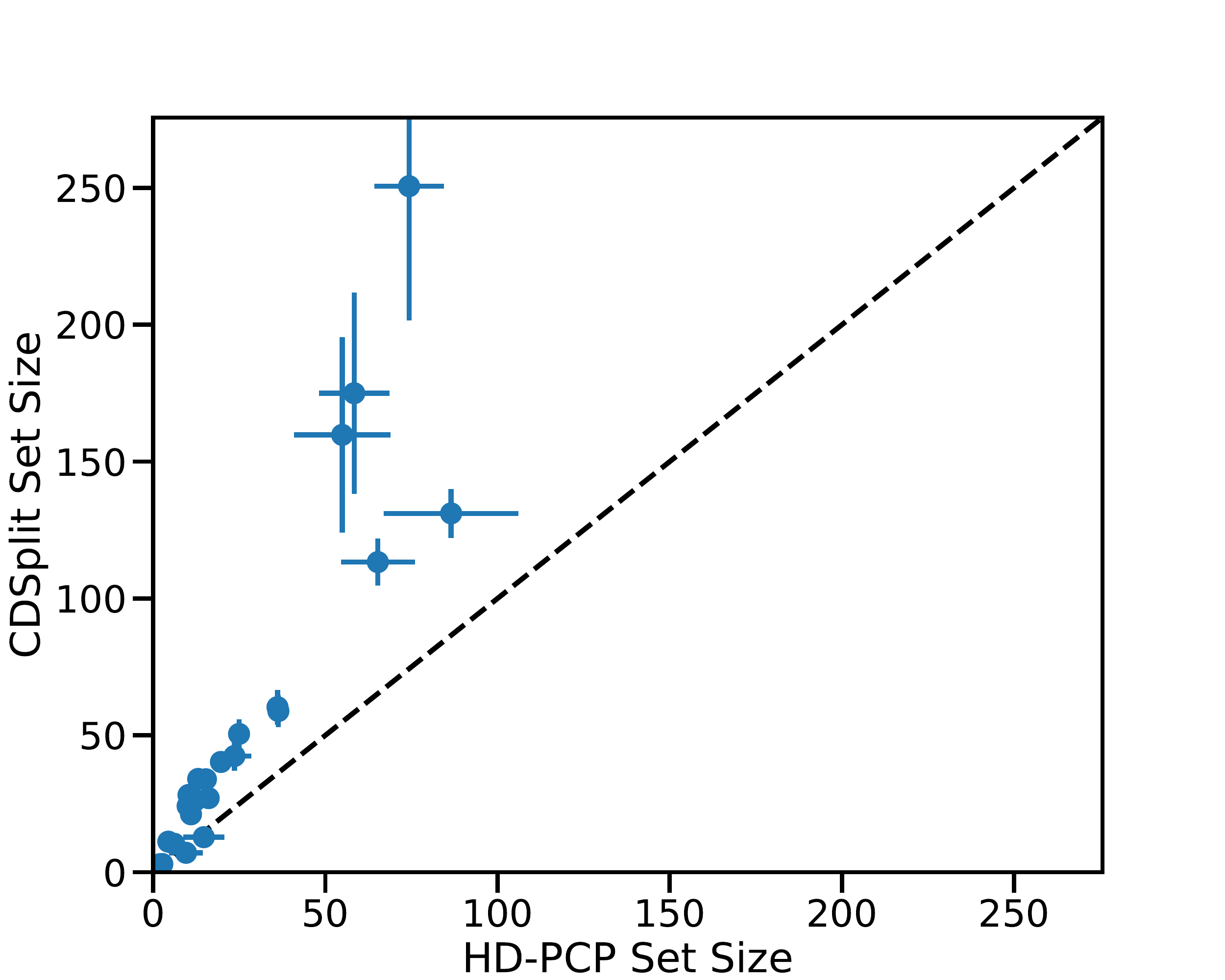}
     \caption{CDSplit vs HD-PCP}
     \end{subfigure} 
     \begin{subfigure}[b]{0.24\textwidth}
         \includegraphics[width=\textwidth]{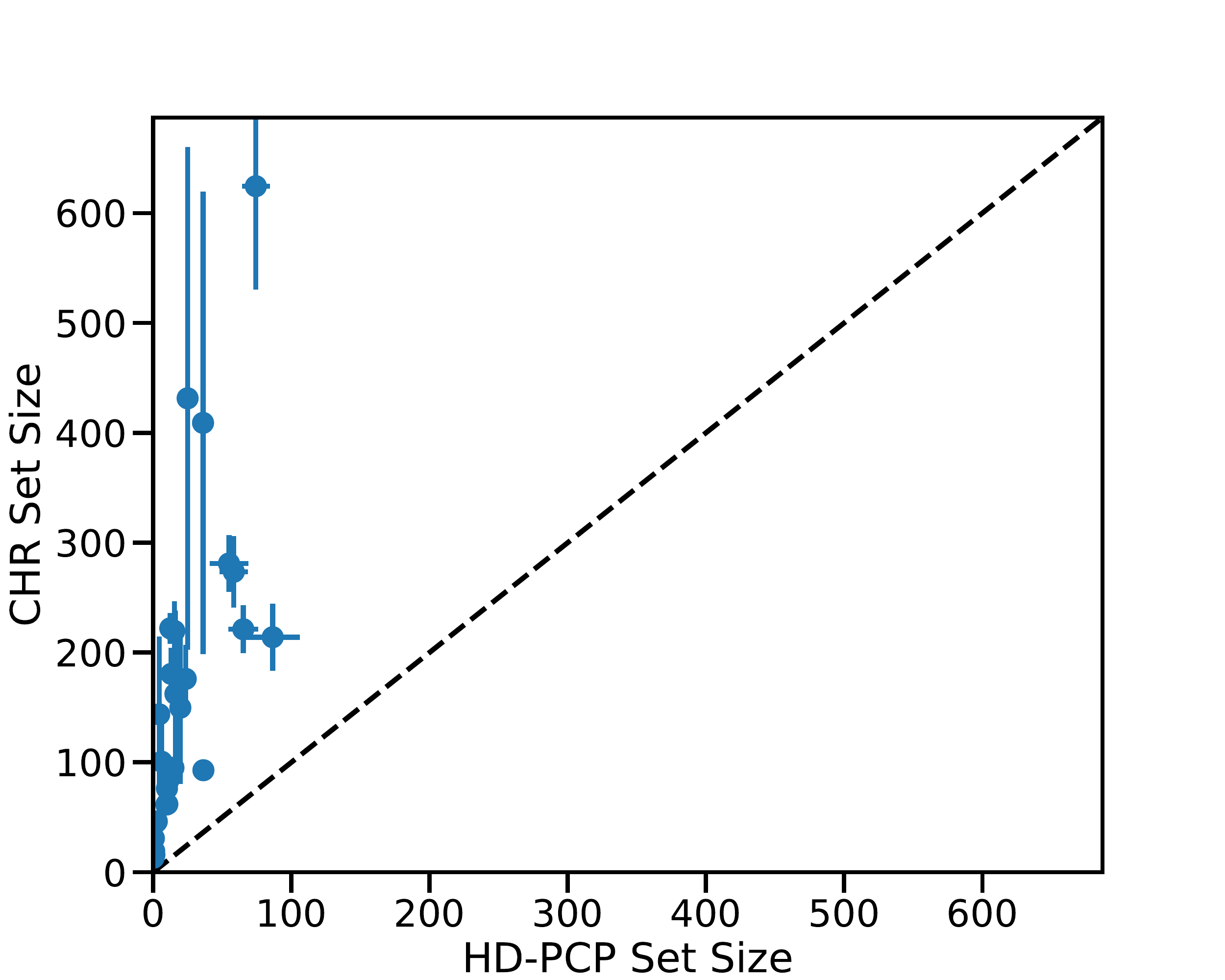}
     \caption{CHR vs HD-PCP}
     \end{subfigure}
     \begin{subfigure}[b]{0.24\textwidth}
         \includegraphics[width=\textwidth]{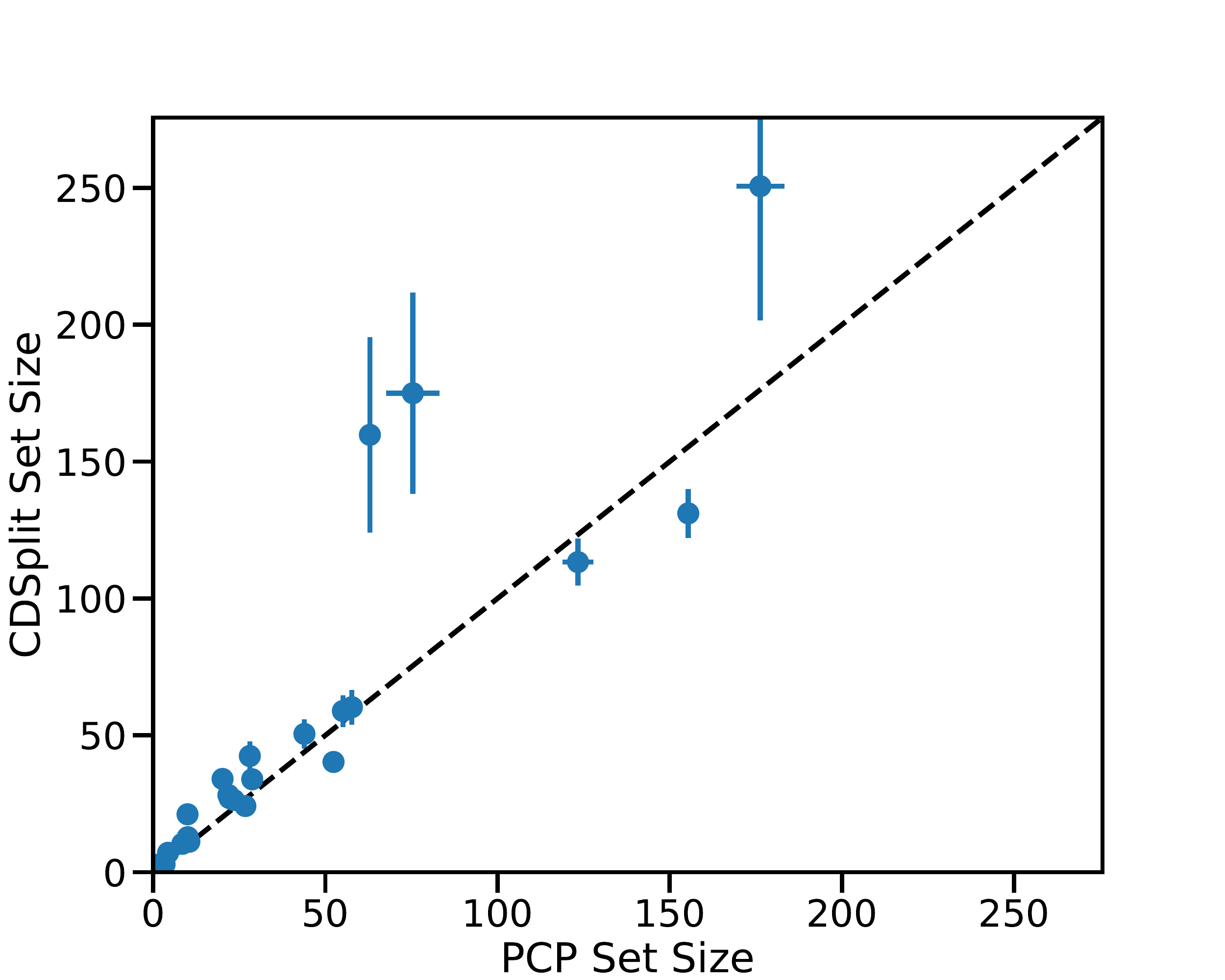}
     \caption{CDSplit vs PCP}
     \end{subfigure}
     \begin{subfigure}[b]{0.24\textwidth}
         \includegraphics[width=\textwidth]{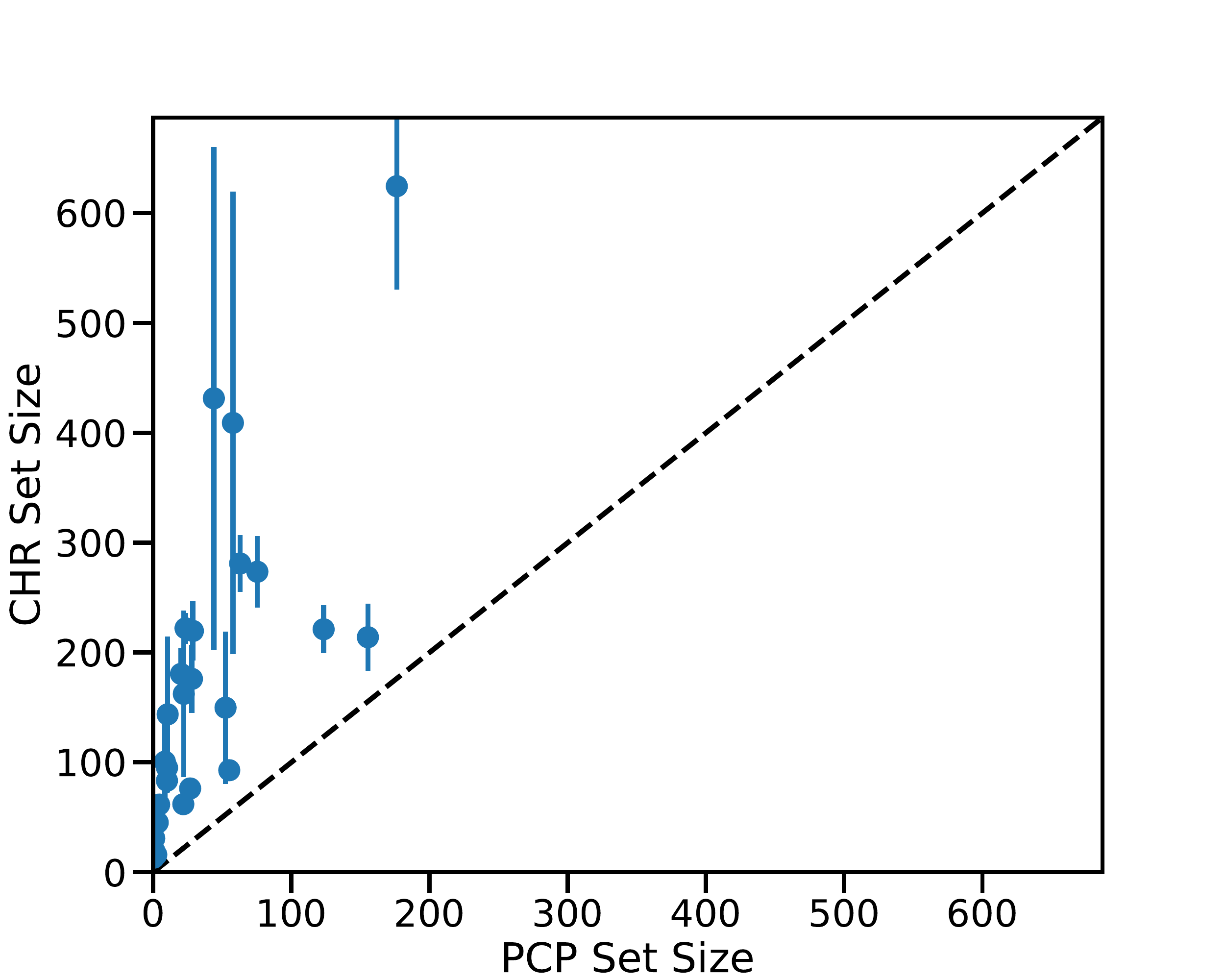}
     \caption{CHR vs PCP}
     \end{subfigure} \\
     \begin{subfigure}[b]{0.24\textwidth}
    \makebox[0pt][r]{\makebox[20pt]{\raisebox{30pt}{\rotatebox[origin=c]{90}{Stock}}}}%
         \includegraphics[width=\textwidth]{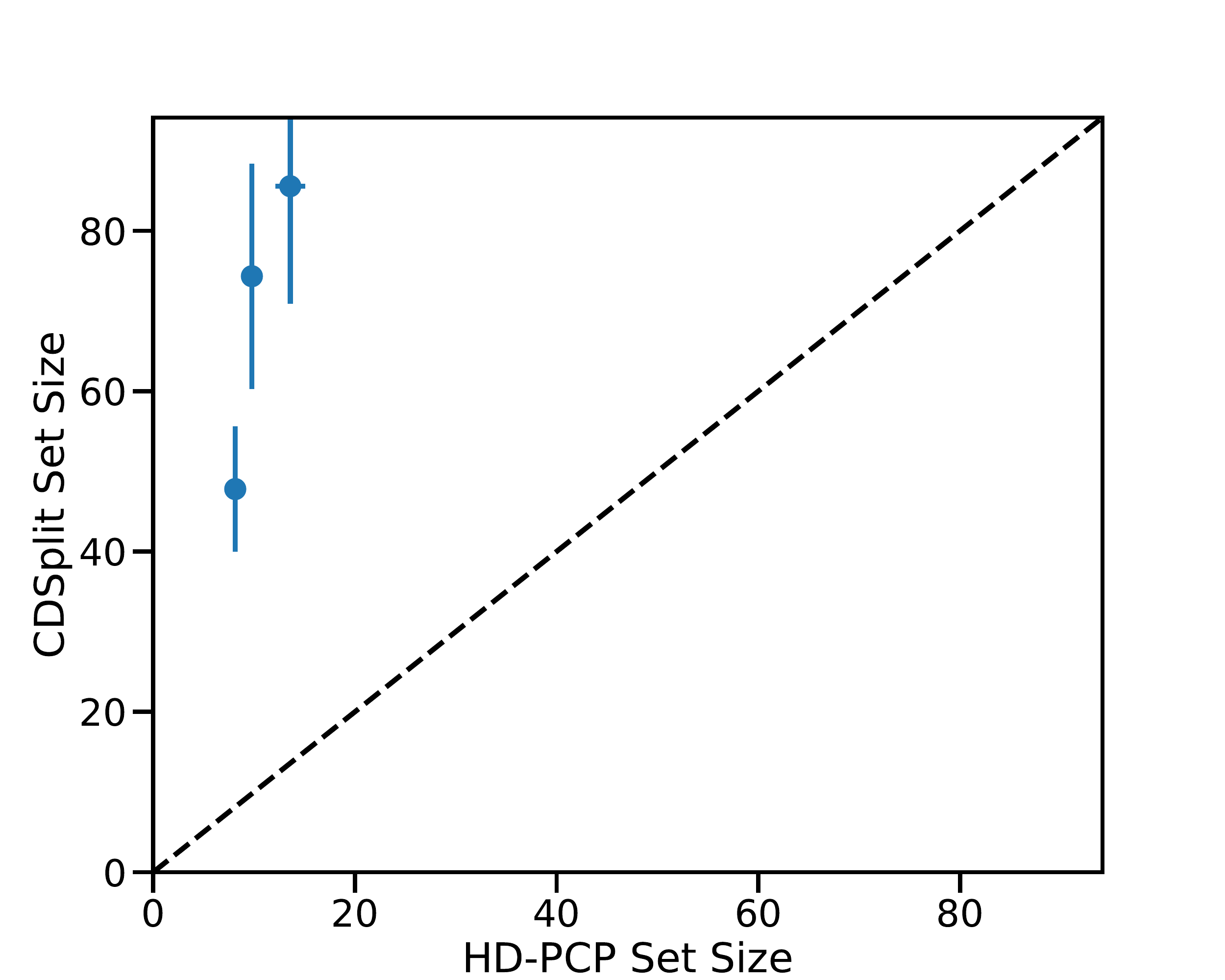}
     \caption{CDSplit vs HD-PCP}
     \end{subfigure} 
     \begin{subfigure}[b]{0.24\textwidth}
         \includegraphics[width=\textwidth]{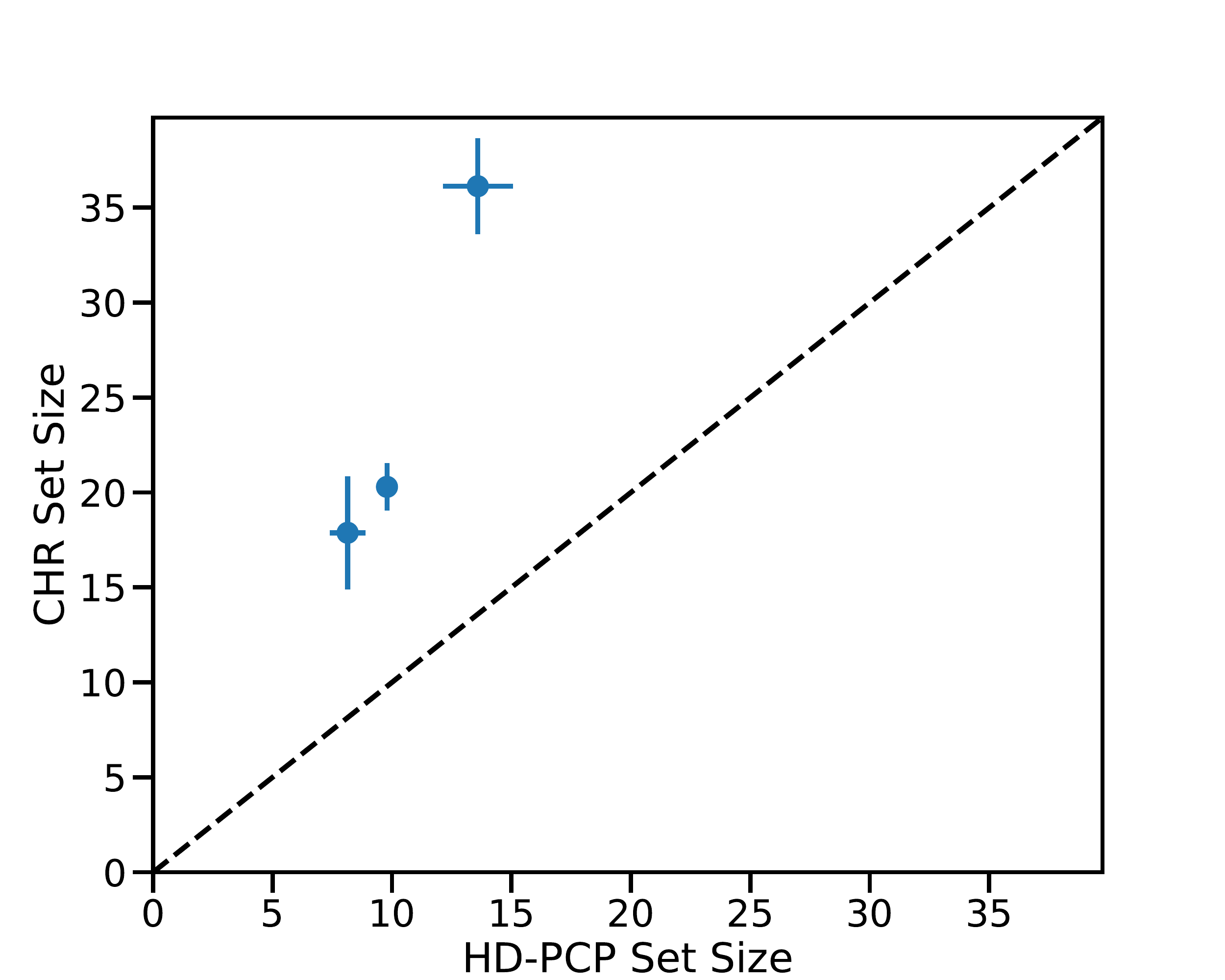}
     \caption{CHR vs HD-PCP}
     \end{subfigure}
     \begin{subfigure}[b]{0.24\textwidth}
         \includegraphics[width=\textwidth]{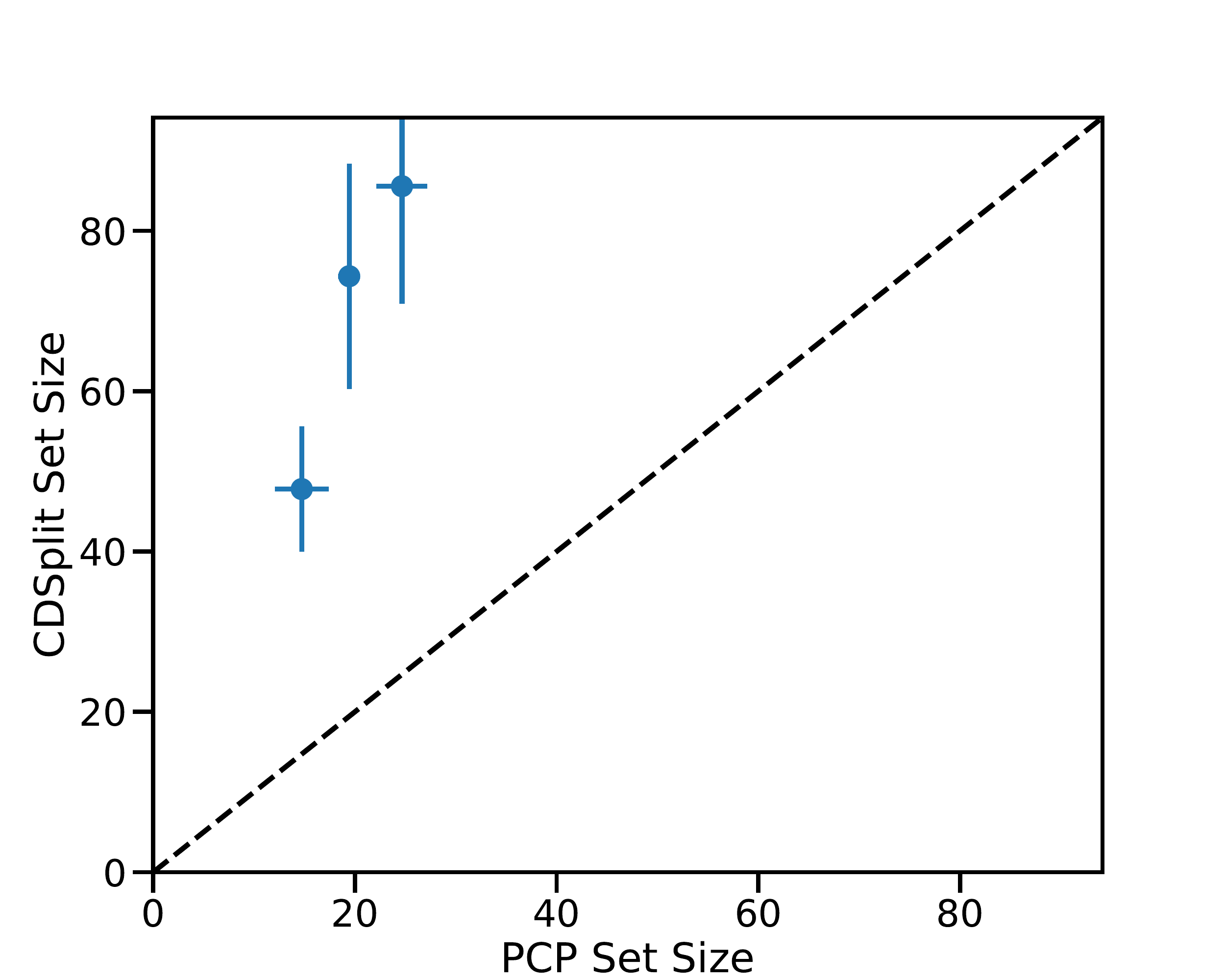}
     \caption{CDSplit vs PCP}
     \end{subfigure}
     \begin{subfigure}[b]{0.24\textwidth}
         \includegraphics[width=\textwidth]{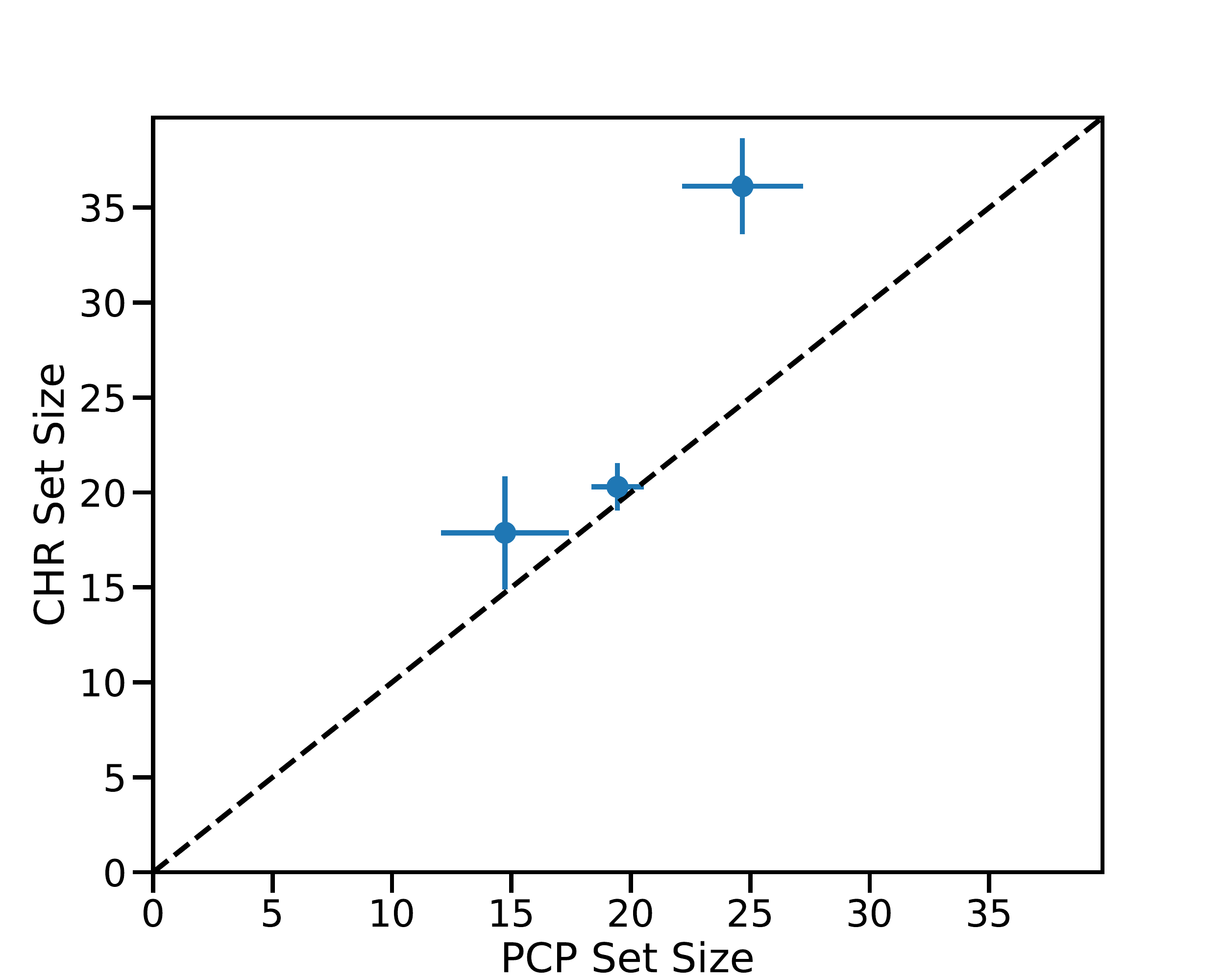}
     \caption{CHR vs PCP}
     \end{subfigure}      
    \caption{Additional Results for Multi-Target Regression Task. Each point corresponds to the size of predictive set for two elements of the target vector. }
    \label{fig:pairwiseapp}
\end{figure}
\end{document}